\newif\ifisarxiv
\isarxivtrue
\documentclass[twoside,11pt]{article}
\ifisarxiv
\usepackage[nohyperref]{sty/jmlr2e-arxiv}
\else
\usepackage{sty/jmlr2e}
\fi
\ShortHeadings{Reverse iterative volume sampling for linear regression}{Derezi\'{n}ski and Warmuth}
\firstpageno{1}
\usepackage[utf8]{inputenc} 
\usepackage[T1]{fontenc}    
\usepackage{amsfonts}       
\usepackage{nicefrac}       
\usepackage{microtype}      
\usepackage{cancel}      
\usepackage{booktabs}       
\usepackage[font=small,labelfont=bf]{caption}

\usepackage{times}
\usepackage{adjustbox}
\usepackage{graphicx} 
\usepackage{color}

\usepackage{amsmath}
\usepackage{verbatim}

\usepackage{algorithm}
\renewcommand{\thealgorithm}{}
\usepackage{algorithmic}
\usepackage{enumitem}

\usepackage{tikz,pgfplots}
\pgfplotsset{compat=newest}

\usepackage{forloop}
\newcounter{loopcntr}

 \definecolor{darkSilver}{cmyk}{0,0,0,0.1}

\usepackage{wrapfig}
\usepackage{eucal}
\numberwithin{equation}{section}
\numberwithin{figure}{section}
\numberwithin{theorem}{section}

\def\xib{\boldsymbol\xi}

\def\val {{\mathrm{val}}}

\def\n{\{1..n\}}

\def\Xbb{\overline{\X}}
\def\ybb{\overline{\y}}

\def\Pc{\mathcal{P}}

\def\Hc{\mathcal{H}}

\newcommand{\Span}{\mathrm{span}}

\newcommand{\ofsub}[1]{\mbox{\small \raisebox{0.0pt}{$(#1)$}}}
\newcommand{\of}[2]{{#1{\!\ofsub{#2}}}}

\newcommand{\fof}[2]{{#1({#2})}}

\newcommand{\lazy}{FastRegVol}
\newcommand{\volsamp}{RegVol}

\newcommand{\Sm}{{S_{-i}}}

\ifx\BlackBox\undefined
\newcommand{\BlackBox}{\rule{1.5ex}{1.5ex}}  
\fi
\renewcommand{\dagger}{+}
\DeclareMathOperator*{\argmin}{\mathop{\mathrm{argmin}}}

\def\x{\mathbf x}

\def\y{\mathbf y}
\def\ybh{\widehat{\mathbf y}}

\def\yh{\widehat{y}}

\def\a{\mathbf a}
\def\b{\mathbf b}
\def\w{\mathbf w}
\def\v{\mathbf v}
\def\m{\mathbf m}

\def\wbt{\widetilde{\mathbf w}}
\def\e{\mathbf e}
\def\zero{\mathbf 0}

\def\u{\mathbf u}

\def\f{\mathbf f}

\def\X{\mathbf X}
\def\Xs{\widetilde{\X}}
\def\B{\mathbf B}
\def\A{\mathbf A}
\def\C{\mathbf C}

\def\F{\mathbf F}

\def\M{\mathbf M}

\def\Z{\mathbf Z}

\def\I{\mathbf I}

\def\A{\mathbf A}
\def\P{\mathbf P}

\def\Rh{\widehat{R}}

\def\E{\mathbb E}

\def\R{\mathbb R} 
\def\tr{\mathrm{tr}}

\def\rank{\mathrm{rank}}

\def\Var{\mathrm{Var}}
\def\Xinv{(\X^\top\X)^{-1}}

\def\xinv{\x_i^\top\Xinv\x_i}

\newcommand{\defeq}{\stackrel{\textit{\tiny{def}}}{=}}

\definecolor{silver}{cmyk}{0,0,0,0.3}
\definecolor{yellow}{cmyk}{0,0,0.9,0.0}
\definecolor{reddishyellow}{cmyk}{0,0.22,1.0,0.0}
\definecolor{black}{cmyk}{0,0,0.0,1.0}
\definecolor{darkYellow}{cmyk}{0.2,0.4,1.0,0}
\definecolor{darkSilver}{cmyk}{0,0,0,0.1}
\definecolor{grey}{cmyk}{0,0,0,0.5}
\definecolor{darkgreen}{cmyk}{0.6,0,0.8,0}
\newcommand{\Red}[1]{{\color{red}  {#1}}}
\newcommand{\Magenta}[1]{{\color{magenta}{#1}}}
\newcommand{\Green}[1]{{\color{darkgreen}  {#1}}}
\newcommand{\Blue}[1]{\color{blue}{#1}\color{black}}
\newcommand{\Brown}[1]{{\color{brown}{#1}\color{black}}}

\newenvironment{proofof}[2]{\par\vspace{2mm}\noindent\textbf{Proof of {#1} {#2}}\ }{\hfill\BlackBox\\[2mm]}

\ifx\proof\undefined
\newenvironment{proof}{\par\noindent{\bf Proof\ }}{\hfill\BlackBox\\[2mm]}
\fi

\ifx\theorem\undefined
\newtheorem{theorem}{Theorem}
\fi

\ifx\example\undefined
\newtheorem{example}{Example}
\fi

\ifx\property\undefined
\newtheorem{property}{Property}
\fi

\ifx\lemma\undefined
\newtheorem{lemma}[theorem]{Lemma}
\fi

\ifx\proposition\undefined
\newtheorem{proposition}[theorem]{Proposition}
\fi

\ifx\remark\undefined
\newtheorem{remark}[theorem]{Remark}
\fi

\ifx\corollary\undefined
\newtheorem{corollary}[theorem]{Corollary}
\fi

\ifx\definition\undefined
\newtheorem{definition}{Definition}
\fi

\ifx\conjecture\undefined
\newtheorem{conjecture}[theorem]{Conjecture}
\fi

\ifx\axiom\undefined
\newtheorem{axiom}[theorem]{Axiom}
\fi

\ifx\claim\undefined
\newtheorem{claim}[theorem]{Claim}
\fi

\ifx\assumption\undefined
\newtheorem{assumption}[theorem]{Assumption}
\fi

\definecolor{darkgreen}{rgb}{0., 0.45, 0.0}

\begin{document}

\title{Reverse iterative volume sampling for linear regression\footnotemark}

\author{%
\name Micha{\l } Derezi\'{n}ski \email mderezin@ucsc.edu\\
\name Manfred K. Warmuth \email manfred@ucsc.edu\\
\addr Department of Computer Science\\
University of California Santa Cruz
}

\editor{}

\maketitle

\begin{abstract}
We study the following basic machine learning task: 
Given a fixed set of input points in $\R^d$ for a
linear regression problem, we wish to predict a hidden response
value for each of the points. We can only afford to attain the
responses for a small subset of the points that are then
used to construct linear predictions for
all points in the dataset. The performance of the
predictions is evaluated by the total square
loss on all responses (the attained as well as the hidden
ones).
We show that a good approximate solution to this least
squares problem can be obtained from just dimension $d$ many
responses by using a joint sampling technique called volume
sampling. Moreover, the least squares solution obtained for the volume
sampled subproblem is an unbiased estimator of optimal solution based
on all $n$ responses. This unbiasedness is a desirable property
that is not shared by other common subset selection techniques. 

Motivated by these basic properties, we develop a theoretical
framework for studying volume sampling, resulting in a number of new
matrix expectation equalities and statistical guarantees 
which are of importance not
only to least squares regression but also to numerical linear
algebra in general. Our methods also lead to a regularized variant of
volume sampling, and we propose the
first efficient algorithms for volume sampling 
which make this technique a practical
tool in the machine learning toolbox. Finally, we provide
experimental evidence which confirms our theoretical findings.

\end{abstract}



\section{Introduction}
\label{sec:introduction}

{\renewcommand{\thefootnote}{\fnsymbol{footnote}}
\footnotetext[1]{This paper is an expanded version of two conference papers
\citep{unbiased-estimates,regularized-volume-sampling}.}} 

As an introductory case, consider linear regression in one dimension. 
We are given $n$ points $x_i$. Each point has a hidden
real response (or target value) $y_i$.
Assume that obtaining the responses is expensive and the learner can
afford to request the responses $y_i$ for only a small number of
indices $i$. 
After receiving the requested responses, the learner determines an
approximate linear least squares solution. In the one
dimensional case, this is just a single weight. 
How many response values does the learner need to request
so that the total square loss of its approximate solution on all
$n$ points is ``close'' to the total loss of the optimal 
linear least squares solution found with the knowledge of
all responses?
We will show here that just \emph{one} response suffices if the index $i$ is chosen
proportional to $x_i^2$. When the learner uses the 
approximate solution $\Blue{w_i^*=\frac{y_i}{x_i}}$, 
then its expected loss equals \Red{2} times the loss of the optimum
$\Green{w^*}$ that is computed based on all responses
(See Figure \ref{f:one}). Moreover, the approximate
solution $\Blue{w_i^*}$ is an unbiased estimator for the optimum
$\Green{w^*}$:
$$\E_i \left[\sum\nolimits_j (x_j \Blue{\frac{y_i}{x_i}} - y_j)^2\right]
= \Red{2}\; \sum\nolimits_j\, (x_j \Green{w^*} -y_j)^2
\quad\text{and}\quad\E_i\left[\Blue{\frac{y_i}{x_i}}\right] = \Green{w^*},
\quad\text{when }P(i)\,\sim\, x_i^2.
$$
We will extend these formulas to higher
dimensions and to sampling more responses 
by making use of a joint sampling
distribution called \emph{volume sampling}. We summarize our
contributions in the next four subsections.
\begin{wrapfigure}{r}{0.4\textwidth}
\vspace{3mm}
\includegraphics[width=0.35\textwidth]{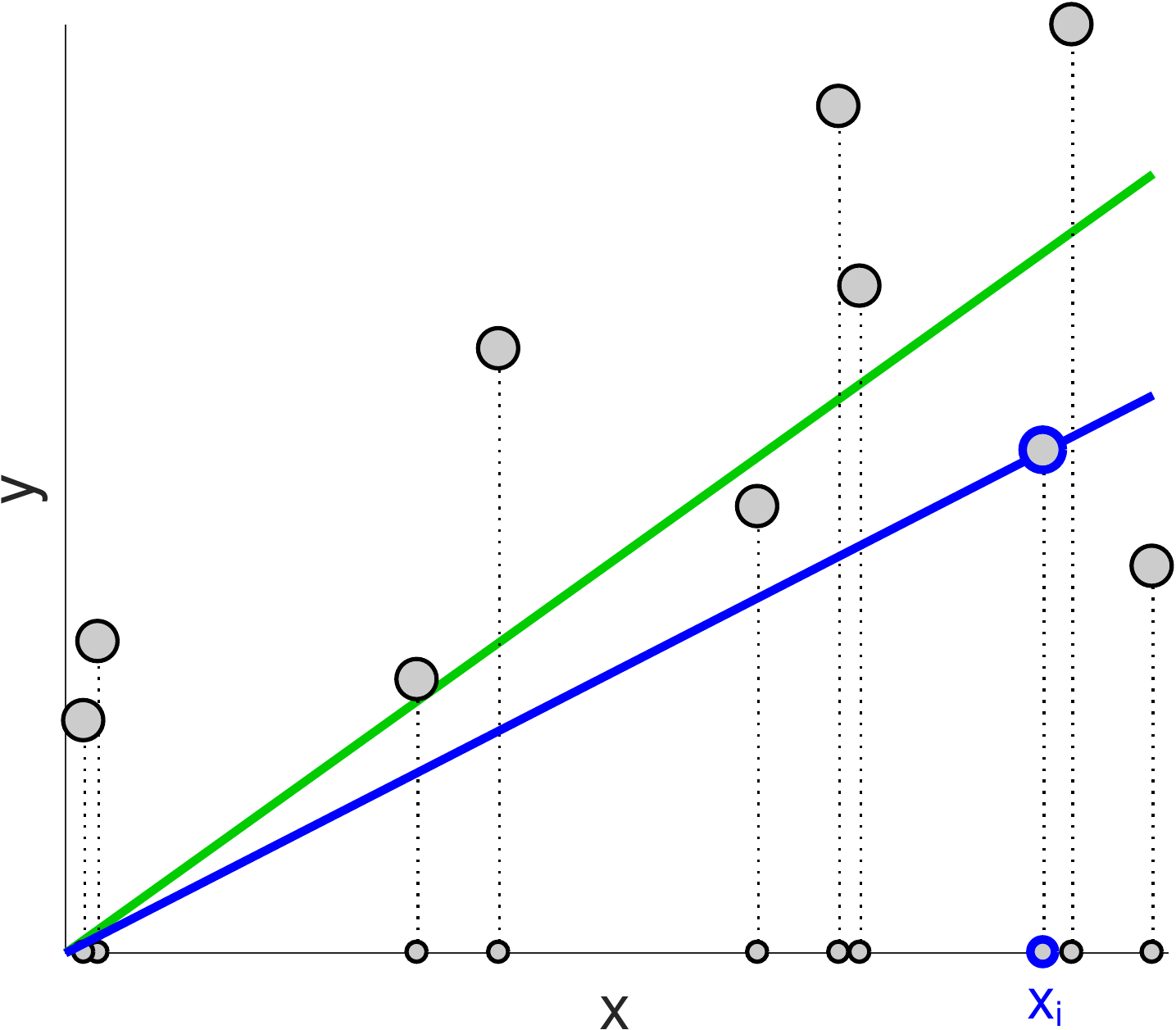}
\caption{The expected loss of $\Blue{w_i^*=\frac{y_i}{x_i}}$
(blue line) based on one response $y_i$ is twice the loss 
of the optimum $\Green{w^*}$ (green line).}
\label{f:one}
\end{wrapfigure}

\paragraph{Least squares with dimension many responses}
Consider the case when the points $\x_i$ lie in
$\R^d$. Let $\X$ denote the $n\times d$ matrix that has the
$n$ transposed points $\x_i^\top$ as rows, 
and let $\y\in \R^n$ be the vector of responses. Now the
goal is to minimize the (total)  square loss 
\[L(\w)=\sum\nolimits_{i=1}^n (\x_i^\top\w-y_i)^2 = \|\X\w - \y\|^2,\]
over all linear weight
vectors $\w\in\R^d$. Let $\w^*$ denote the optimal such weight vector.
We want to minimize the square loss based on a small number of responses
we attained for a subset of rows.
Again, the learner is initially given the fixed set of $n$ rows
(i.e. fixed design), but none of the responses.
It is then allowed to choose a random subset of $d$ indices,
$S\subseteq \{1..n\}$, and obtains the responses for the
corresponding $d$ rows.
The learner proceeds to find the optimal linear least
squares solution $\of{\w^*}S$ for the subproblem $(\X_S,\y_S)$.
where $\X_S$ is the subset of $d$ rows of $\X$ indexed by $S$
and $\y_S$ the corresponding $d$ responses from the
response vector $\y$.
As a generalization of the one-dimensional
distribution that chooses an index based on the squared
length, set $S$ of size $d$ is chosen proportional to the
squared volume of the parallelepiped spanned by the rows of $\X_S$.
This squared volume equals $\det(\X_S^\top\X_S)$.
Using elementary linear algebra, we will show that
{\em volume sampling} the set $S$ assures that
$\of{\w^*}S$ is a good approximation to $\w^*$ in the following
sense:
In expectation, the square loss (on all $n$ row response
pairs) of $\of{\w^*}S$ is equal $\Red{d+1}$ times the
square loss of $\w^*$:
\begin{align*}
\E[L(\of{\w^*}S)]  = \Red{(d+1)}\,L(\w^*)
,\quad \text{when }P(S) \sim \det(\X_S^\top\X_S).
\end{align*}
Furthermore, for any sampling procedure that attains less than
$d$ responses, the ratio between the expected loss and the loss 
of the optimum cannot be bounded by a constant.

\begin{figure}[b]
\vspace{-3mm}
\addtolength{\tabcolsep}{2pt} 
\centerline{
\begin{tabular}{ccccc}
\begin{tikzpicture}[scale=0.45]
    \draw [fill=brown!30] (0,0) rectangle (2,4);
    \draw [color=black] (0,2.5) -- (2,2.5);
    \draw (1,2) node {\mbox{\footnotesize $\x_i^\top$}}; 
    \draw [decorate,decoration={brace}] (-.1,0) -- (-.1,4);
    \draw (-.55,2.05) node {\mbox{\fontsize{8}{8}\selectfont $n$}}; 
    \draw [decorate,decoration={brace}] (0,4.1) -- (2,4.1);
    \draw (1,4.6) node {\mbox{\fontsize{8}{8}\selectfont $d$}}; 
\end{tikzpicture} 
&
\begin{tikzpicture}[scale=0.45]
    \draw (0,0) rectangle (4,4);
    \draw [color=blue] (.45,3.5) -- (2.7,1.3);
    \draw (1.75,3) node {\mbox{\footnotesize $\Blue{S}$}}; 
\end{tikzpicture} 
&
\begin{tikzpicture}[scale=0.45]
    \draw (0,0) rectangle (2,4);
    \draw[fill=blue!30] (0,1.3) rectangle (2,3.5);
    \draw (1,2.55) node {\mbox{\footnotesize $\Blue{\X_S}$}};
    \draw [decorate,decoration={brace}] (-.1,1.3) -- (-.1,3.5);
    \draw (-.5,2.45) node {\mbox{\fontsize{8}{8}\selectfont $s$}}; 
\end{tikzpicture} 
&
\begin{tikzpicture}[scale=0.45]
    \draw [fill=brown!30] (0,0) rectangle (4,2);
\end{tikzpicture} 
&
\begin{tikzpicture}[scale=0.45]
    \draw (0,0) rectangle (4,2);
    \draw[fill=blue!30] (.5,0) rectangle (2.7,2);
    \draw (1.57,.95) node {\mbox{\footnotesize $\;\Blue{(\!\X_S\!)^{\!+}}$}};
\end{tikzpicture} 
\\[-.1cm]
\mbox{\footnotesize $\quad\ \X$}
&
\mbox{\footnotesize $\I_S$}
&
\mbox{\footnotesize $\quad\ \I_S\X$}
&
\mbox{\footnotesize $\X^+$}
&
\mbox{\footnotesize$(\I_S\X)^+$}
\end{tabular}
}
\addtolength{\tabcolsep}{-2pt} 
\vspace{-2mm}
  \caption{Shapes of the matrices. 
           The indices of $S$ may not be consecutive.}
\label{f:shapes}
\end{figure}
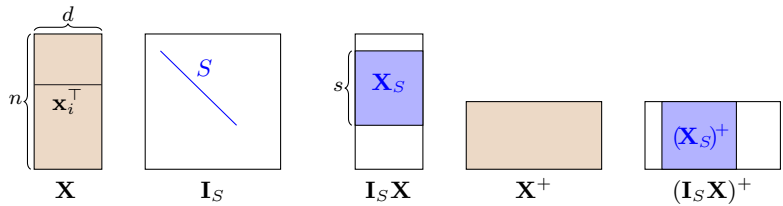
\paragraph{Unbiased pseudoinverse estimator} 
There is a direct connection between solving
linear least squares problems  and the pseudoinverse $\X^+$ of matrix
$\X$: For an $n-$dimensional response vector $\y$, 
the optimal solution is $\w^*=\argmin_\w ||\X\w-\y||^2=\X^+\y$.
Similarly $\of{\w^*}S=(\X_S)^{+}\y_S$ is the solution
for the subproblem $(\X_S,\y_S)$.
We propose a new implementation of volume sampling called
{\em reverse iterative sampling} which
enables a novel proof technique for obtaining elementary
expectation formulas for pseudoinverses based on volume
sampling.

Suppose that our goal is to estimate the pseudoinverse $\X^+$ 
based on the pseudoinverse of a subset of rows. Recall
that for a subset $S\subseteq \{1..n\}$ of $s$ row indices
(where the size $s$ is fixed and $s\geq d$), 
we let $\X_S$ be the submatrix of the $s$
rows indexed by $S$ (see Figure \ref{f:shapes}).
Consider a version of $\X$ in which
all but the rows of $S$ are zero. 
This matrix equals $\I_S\X$, where the selection matrix $\I_S$ is an $n$-dimensional diagonal 
matrix with $(\I_S)_{ii}=1$ if $i\in S$ and 0 otherwise.

For the set $S$ of fixed size $s\ge d$ row indices 
chosen proportional to $\det(\X_S^\top\X_S)$, 
we can prove the following expectation formulas:
\vspace{-1mm}
\begin{align*}
\E[(\I_S\X)^+]=\X^+ \quad \text{and}\quad
\E[\underbrace{(\X_S^\top\X_S)^{-1}}_
{(\I_S\X)^{+}(\I_S\X)^{+\top}} ]= \frac{n-d+1}{s-d+1}\,
\underbrace{(\X^\top\X)^{-1}}_{\X^+\X^{+\top}}.
\end{align*}
\\[-2mm]
Note that $(\I_S\X)^+$ has the $d\times n$ shape of $\X^+$
where the $s$ columns indexed by $S$ contain $(\X_S)^+$ and the
remaining $n-s$ columns are zero.
The expectation of this matrix is $\X^+$ even though
$(\X_S)^+$ is clearly not a submatrix of $\X^+\!$.
This expectation formula now implies that for any size $s\ge d$, if
$S$ of size $s$ is drawn by volume sampling, then $\of{\w^*}S$ is an
unbiased estimator%
\footnote{For size $s=d$ volume sampling, 
the fact that $\E[\w^*\!(S)]=\w^*$
can be found in an early paper \citep{bental-teboulle}.
They give a direct proof based on Cramer's rule. 
}
for $\w^*\!$, i.e. 
$$\E[\of{\w^*}S]=\E[(\X_S)^+\y_S]
=\E[(\I_S\X)^+\y]
=\E[(\I_S\X)^+]\,\y
=\X^+\y =\w^*\!.$$

The second
expectation formula can be viewed as a second moment of the
pseudoinverse estimator $(\I_S\X)^+\!$, and it can be used to compute
a useful notion of matrix variance with applications in random matrix
theory:
\vspace{-1mm}
\[\E[(\I_S\X)^{+}(\I_S\X)^{+\top}] - \E[(\I_S\X)^+] \E[(\I_S\X)^+]^\top
  = \frac{n-s}{s-d+1}\X^+\X^{+\top}\!. \]
\\[-3mm]
\paragraph{Regularized volume sampling}
We also develop a new regularized variant of volume sampling, which
extends reverse iterative sampling to selecting subsets of size
smaller than $d$, and leads to a useful extension of the above matrix
variance formula.  Namely,
for any $\lambda\geq 0$, our $\lambda$-regularized
procedure for sampling subsets $S$ of size $s$ satisfies
\begin{align*}
\E\big[ (\X_S^\top\X_S+\lambda\I)^{-1}\big] \preceq
  \frac{n-d_\lambda+1}{s-d_\lambda+1}(\X^\top\X+\lambda\I)^{-1},
\end{align*}
where $d_\lambda\defeq \tr(\X(\X^\top\X+\lambda\I)^{-1}\X^\top)\leq d$
is a standard notion of statistical dimension. Crucially, the above
bound holds for subset sizes $s\geq d_\lambda$, which can be much
smaller than the dimension $d$.

Under the additional assumption that response vector $\y$ 
is generated by a linear transformation
distorted with bounded white noise, 
the expected bound on $(\X_S^\top\X_S+\lambda\I)^{-1}$ leads to
strong variance bounds for ridge regression estimators. 
Specifically, we prove that
when $\y=\X\wbt + \xib$, with $\xib$ having mean zero and bounded
variance $\Var[\xib]\preceq \sigma^2\I$, then if $S$ is sampled according 
to $\lambda$-regularized volume sampling with $\lambda\leq
\frac{\sigma^2}{\|\wbt\|^2}$, we can obtain the following mean squared
prediction error (MSPE) bound:
\begin{align*}
\E_S\E_{\xib}\bigg[\frac{1}{n}\|\X(\of{\w_\lambda^*}S-\wbt)\|^2\bigg]&
\leq\frac{\sigma^2d_\lambda}{s-d_\lambda+1},
\end{align*}
where $\of{\w_\lambda^*} S = (\X_S^\top\X_S+\lambda\I)^{-1}\X_S^\top\y_S$
is the ridge regression estimator for the subproblem $(\X_S,\y_S)$.
Our new lower bounds show
that the above upper bound for regularized volume sampling 
is essentially optimal with respect to the choice of a subsampling
procedure. 

\paragraph{Algorithms and experiments}
The only
known polynomial time algorithm for size $s>d$ volume sampling was
recently proposed by \cite{dual-volume-sampling} with time complexity
$O(n^4 s)$. 
In this paper we give two new algorithms
using our general framework of reverse iterative sampling: 
one with deterministic
runtime of $O((n\!-\!s\!+\!d)nd)$, and a second one that with high
probability finishes in time $O(nd^2)$. Thus both algorithms
improve on the state-of-the-art by a factor of at least $n^2$
and make volume sampling nearly as efficient as the comparable i.i.d. sampling
technique called leverage score sampling. Our experiments on real
datasets confirm the efficiency of our algorithms, and show that for
small sample sizes $s$ volume sampling is more effective than leverage
score sampling for the task of subset selection for linear regression.

\paragraph{Related work}
\label{sec:related-work}
Volume sampling is a type of determinantal point
process (DPP) \citep{dpp}. DPP's have been given a lot of attention in
the literature with many applications to machine learning, including
recommendation systems \citep{dpp-shopping} and clustering
\citep{dpp-clustering}. Many exact and approximate methods for efficiently
generating samples from this distribution have been proposed
\citep{efficient-volume-sampling,k-dpp}, making it a useful tool in
the design of randomized algorithms. Most of those methods focus on
sampling $s\leq d$ elements. In this paper, we study volume sampling
sets of size $s\geq d$, which was proposed by
\cite{avron-boutsidis13} and motivated with applications in graph
theory, linear regression, matrix approximation and more. 

The problem of selecting a subset of the rows of the input
matrix for solving a linear
regression task has been extensively studied in statistics literature
under the terms {\em optimal design} \citep{optimal-design-book} and
{\em pool-based active learning}
\citep{pool-based-active-learning-regression}. Various 
criteria for subset selection have been proposed, like A-optimality
and D-optimality. For example, A-optimality seeks to minimize
$\tr((\X_S^\top\X_S)^{-1})$, which is combinatorially hard to optimize
exactly. We show that for size $s\ge d$ volume sampling,
$\E[(\X_S^\top\X_S)^{-1}] = \frac{n-d+1}{s-d+1}\,(\X^\top\X)^{-1}$,
which provides an approximate randomized solution of
the sampled inverse covariance matrix rather than just its
trace.

In the field of computational geometry a variant of volume sampling
was used to obtain optimal bounds for low-rank matrix approximation. In this
task, the goal is to select a small subset of rows of a matrix
$\X\in \R^{n\times d}$ (much fewer than the rank of $\X$, which is
bounded by $d$), so that a good low-rank approximation of $\X$ can be
constructed from those rows. \cite{pca-volume-sampling} showed that 
volume sampling of size $s<d$ index sets 
obtains optimal multiplicative bounds for this task 
and polynomial time algorithms for size $s<d$ volume
sampling were given in \cite{efficient-volume-sampling} and
\cite{more-efficient-volume-sampling}.
We show in this paper that for linear regression, fewer than rank many 
rows do not suffice to obtain multiplicative bounds.
This is why we focus on volume sampling sets of size
$s\geq d$ (recall that, for simplicity, we assume that $\X$ is full rank).

Computing approximate solutions to linear regression has
been explored in the domain of numerical linear algebra (see
\cite{randomized-matrix-algorithms} for an overview). 
Here, multiplicative bounds on the loss of the approximate solution can be
achieved via two approaches. The first approach relies on sketching
the input matrix $\X$ and the response vector $\y$ by
multiplying both by the same suitably chosen random matrix. 
Algorithms which use sketching to generate a smaller
input matrix for a given linear regression problem are
computationally efficient
\citep{sarlos-sketching,regression-input-sparsity-time}, but they
require all of the responses from the original problem
to generate the sketch and are thus not suitable for the goal 
of using as few response values as possible.
The second approach is based on subsampling the rows of
the input matrix and only asking for the responses of the sampled rows. 
The learner optimally solves the sampled subproblem\footnote{Note 
that those methods typically require
additional rescaling of the subproblem, whereas the techniques
proposed in this paper do not require any rescaling.}
and then uses the obtained weight vector for
its prediction on all rows. The selected subproblem is
known under the term ``$\b$-agnostic minimal coreset'' 
in \citep{coresets-regression,cur-decomposition} 
since it is selected without knowing the response vector (denoted as
the vector $\b$).
The second approach coincides with the goals of this paper
but the focus here is different in a number of ways.
First, we focus on the smallest sample size for which a
multiplicative loss bound is possible:
Just $d$ volume sampled rows are
sufficient to achieve a multiplicative bound with a fixed
factor, while $d-1$ are not sufficient. 
A second focus here is the efficiency and the combinatorics of volume
sampling. The previous work is mostly based on 
i.i.d. sampling using the statistical leverage scores \citep{fast-leverage-scores}.
As we show in this paper, leverage scores are the
marginals of volume sampling and
any i.i.d. sampling method requires 
sample size $\Omega(d\log d)$ to achieve multiplicative loss bounds
for linear regression. On the other hand, the rows obtained
from volume sampling are {\em selected jointly} 
and this makes the chosen subset more informative and
brings the required sample size down to $d$.
Third, we focus on the fact that the estimators produced
from volume sampling are unbiased and therefore can be
averaged to get more accurate estimators.
Using our methods, averaging immediately leads to an unbiased estimator
with expected loss $1+\epsilon$ times the optimum based on
sampling $d^2/\epsilon$ responses in total. We leave it as
an open problem to construct a $1+\epsilon$ factor
unbiased estimator from sampling only $O(d/\epsilon)$ responses. 
If unbiasedness is not a concern, then such an estimator has
recently been found \citep{chen2017condition}.

\paragraph{Outline of the paper}
In the next section, we define volume sampling as an instance of a
more general procedure we call reverse iterative sampling, and we use
this methodology to prove closed form matrix expressions
for the expectation of the pseudoinverse estimator $(\I_S\X)^+$
and its square $(\I_S\X)^+(\I_S\X)^{+\top}$, when $S$ is sampled by volume sampling.
Central to volume sampling is the Cauchy-Binet formula for
determinants. As a side, we produce a number of short self-contained
proofs for this formula and show that leverage scores are
the marginals of volume sampling.
Then in Section \ref{sec:linear-regression} we formulate the problem of 
solving linear regression from a small number of responses, 
and state the upper bound for the
expected square loss of the volume sampled least squares estimator
(Theorem \ref{t:loss}),
followed by a discussion and related lower-bounds. 
In Section \ref{sec:proof-loss}, we prove Theorem \ref{t:loss}
and an additional related matrix expectation formula. 
We next discuss in Section \ref{sec:av} how unbiased estimators can easily be
averaged for improving the expected loss and discuss 
open problems for constructing unbiased estimators.
A new regularized
variant of volume sampling is proposed in Section \ref{sec:regularized},
along with the statistical guarantees it offers for computing subsampled ridge
regression estimators. Next, we present
efficient volume sampling algorithms in Section \ref{sec:algorithm},
based on the reverse iterative sampling paradigm, which are then
experimentally evaluated in Section \ref{sec:experiments}.
Finally, Section
\ref{sec:conclusions} concludes the paper by suggesting a future
research direction.

\section{Reverse iterative sampling}
\label{sec:pseudo-inverse}
Let $n$ be an integer dimension.
For each subset $S\subseteq \{1..n\}$ of size $s$ we are given a 
matrix formula $\fof{\F}S$.
Our goal is to sample set $S$ of size $s$ using some
sampling process and then develop concise expressions for
$\E_{S:|S|=s}[\fof{\F}S]$. Examples of formula
classes $\fof{\F}S$ will be given below.

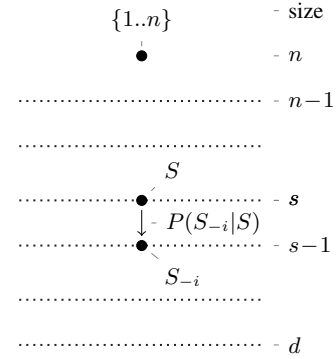
\begin{wrapfigure}{r}{0.37\textwidth}
\centering
\begin{tikzpicture}[font=\footnotesize,scale=.95,pin distance=.7mm]
    \begin{axis}[hide axis,ymin=-3,ymax = 6,xmax=2,xmin=-2]
	\addplot[mark=none, dotted, black,style=thick] coordinates {(-1,3.2) (1,3.2)};
	\addplot[mark=none, dotted, black,style=thick] coordinates {(-1,2.2) (1,2.2)};
	\addplot[mark=none, dotted, black,style=thick] coordinates {(-1,1) (1,1)};
	\addplot[mark=none, dotted, black,style=thick] coordinates {(-1,0) (1,0)};
	\addplot[mark=none, dotted, black,style=thick] coordinates {(-1,-1.2) (1,-1.2)};
	\addplot[mark=none, dotted, black,style=thick] coordinates {(-1,-2.2) (1,-2.2)};

        \addplot[mark=*] coordinates {(0,4.2)} node[pin=90:{$\{1..n\}$}]{};
        \addplot[mark=*] coordinates {(0,1)} node[pin=45:{$S$}]{};
        \addplot[mark=*] coordinates {(0,0)} node[pin=-45:{$S_{-i}$}]{};
        \addplot[mark={}] coordinates {(0,0.5)} node[pin=0:{$P(S_{-i}|S)$}]{};
        \node[draw=none] (a) at (0,1) {};
        \node[draw=none] (b) at (0,0) {};
        \draw[->] (a)--(b);

        \addplot[mark={}] coordinates {(1,5.2)} node[pin=0:{size}]{};
        \addplot[mark={}] coordinates {(1,4.2)} node[pin=0:{$n$}]{};
        \addplot[mark={}] coordinates {(1,3.2)} node[pin=0:{$n\!-\!1$}]{};
        \addplot[mark={}] coordinates {(1,1)} node[pin=0:{$s$}]{};
        \addplot[mark={}] coordinates {(1,1)} node[pin=0:{$s$}]{};
        \addplot[mark={}] coordinates {(1,0)} node[pin=0:{$s\!-\!1$}]{};
        \addplot[mark={}] coordinates {(1,-2.2)} node[pin=0:{$d$}]{};
    \end{axis}
    \end{tikzpicture}
\caption{Reverse iterative sampling.}
\label{fig:dag-sampling}
\end{wrapfigure}

We represent the sampling by a directed acyclic graph (DAG), with a
single root node corresponding to the full set $\{1..n\}$. Starting
from the root, we proceed along the edges of the graph, 
iteratively removing elements from the set $S$ (see Figure
\ref{fig:dag-sampling}). 
Concretely, consider a DAG with levels $s = n, n-1, ..., d$.
Level $s$ contains $n \choose s$ nodes 
for sets $S\subseteq \{1..n\}$ of size $s$. 
Every node $S$ at level $s>d$ has $s$ directed edges
to the nodes $S-\{i\}$ (also denoted $\Sm$) at the next lower level. 
These edges are labeled with a conditional probability 
vector $P(S_{-i}|S)$.
The probability of a (directed) path is the product of the
probabilities along its edges.
The outflow of probability from each node on
all but the lowest level $d$ is 1. 
We let the probability
$P(S)$ of node $S$ be the probability of all paths 
from the top node $\{1..n\}$ to $S$ and set the
probability $P(\{1..n\})$ of the top node to 1.
We associate a formula $\fof{\F}S$ with each set node $S$ in the
DAG. The following key equality lets us compute
expectations.

\begin{lemma}
\label{l:key}
If for all $S\subseteq \{1..n\}$ of size greater than $d$ we have
$$\Blue{\fof{\F}{S}=\sum_{i\in S} P(\Sm|S)\fof{\F}{S_{-i}}},$$
then for any $s\in\{d..n\}$: $\;\;\E_{S:|S|=s} [\fof{\F}{S}]
=\sum_{S:|S|=s} P(S)\fof{\F}{S} = \fof{\F}{\{1..n\}}.$
\end{lemma}
\begin{proof}
Suffices to show that expectations at successive layers $s$
and $s-1$ are equal for $s>d$:
\begin{align*}
\sum_{S:|S|=s} P(S) \,\Blue{\fof{\F}{S}}
= \!\!\!\sum_{S:|S|=s} P(S) \Blue{\sum_{i\in S} P(S_{-i}|S) \,\fof{\F}{S_{-i}}}
&= \sum_{S:|S|=s} \;\,\sum_{i\in S} P(S) P(S_{-i}|S) \fof{\F}{S_{-i}}
\\
&=\!\!\!\sum_{T:|T|=s-1} \underbrace{\sum_{j\notin T} P(T_{+j}) P(T|T_{+j})}_{P(T)} \fof{\F}{T}.
\end{align*}
Note that the r.h.s. of the first line has one summand per
edge leaving level $s$, and the r.h.s. of the second line has one
summand per edge arriving at level $s-1$. 
Now the last equality holds because
the edges leaving level $s$ are exactly those
arriving at level $s-1$, and the summand for each edge in
both expressions is equivalent.
\end{proof}

\subsection{Volume sampling}\label{sec:vol-def}
Given a tall full rank matrix $\X\in\R^{n\times d}$
and a sample size $s\in \{d..n\}$,
volume sampling chooses subset $S\subseteq\{1..n\}$
of size $s$ with probability proportional to squared volume spanned by
the columns of submatrix\footnote{For sample size $s=d$,
the rows and columns of $\X_S$ have the same length and
$\det(\X_S^\top\X_S)$ is also the squared
volume spanned by the rows $\X_S$.}
 $\X_S$ and this squared volume equals $\det(\X_S^\top\X_S)$. 
The following theorem uses
the above DAG setup to compute the normalization constant for this
distribution. 
Note that all subsets $S$ of volume 0 will be ignored, 
since they are unreachable in the proposed sampling procedure.
\begin{theorem}
\label{t:vol}
Let $\X\in \R^{n\times d}$, where $d\le n$ and $\det(\X^\top\X)>0$.
For any set $S$ of size $s>d$ for which
$\det(\X_S^\top\X_S)>0$, 
define the probability of the edge from $S$ to $\Sm$ for $i\in S$ as:
\begin{align*}
\!\!
P(\Sm|S)\!\defeq\!\frac{\det(\X_\Sm^\top\X_\Sm)}
             {\Red{(s\!-\!d)}\det(\X_S^\top\X_S)}
\!=\! \frac{1\!-\!\x_i^\top (\X_S^\top\X_S)^{-1}\x_i} {s\!-\!d},
\tag{\bf reverse iterative volume sampling}
\end{align*}
where $\x_i$ is the $i$th row of $\X$.
In this case $P(\Sm|S)$ is a proper probability distribution.
If $\det(\X_S^\top\X_S)=0$, then simply set $P(\Sm|S)$ to $\frac{1}{s}$.
With these definitions, $\sum_{S:|S|=s} P(S)=1$ for all $s\in\{d..n\}$
and the probability of all paths from
the root to any subset $S$ of size at least $d$ is
\begin{align*}
P(S)= \frac{\det(\X_S^\top\X_S)} {{n-d \choose s-d}\det(\X^\top\X)}.
\tag{\bf volume sampling}
\end{align*}
\end{theorem}
The rewrite of the ratio
$\frac{\det(\X_\Sm^\top\X_\Sm)}
      {\det(\X_S^\top\X_S)}$
as $1-\x_i^\top (\X_S^\top\X_S)^{-1}\x_i$
is Sylvester's Theorem for determinants. Incidentally, this is the
only property of determinants used in this section.

The theorem also implies a 
generalization of the Cauchy-Binet formula to size $s\ge d$ sets: 
\begin{equation}
\label{e:cbgen}
\sum_{S:|S|=s}\det(\X_S^\top\X_S)={n-d\choose s-d} \det(\X^\top\X).
\end{equation}
When $s=d$, then the binomial coefficient is 1 and the
above becomes the vanilla Cauchy-Binet formula.
The below proof of the theorem thus results in a minimalist proof of this
classical formula as well. The proof uses the reverse iterative
sampling (Figure \ref{fig:dag-sampling}) and the fact that
all paths from the root to node $S$ have the same
probability. For the sake of completeness we also give a
more direct inductive proof of the above generalized Cauchy-Binet
formula in Appendix \ref{a:CB}.

\noindent
\begin{proof}
First, for any node $S$ s.t. $s>d$ and $\det(\X_S^\top\X_S)>0$,
the probabilities out of $S$ sum to 1:
$$\sum_{i\in S} P(\Sm|S)=\sum_{i\in S}
\frac{1-\tr((\X_S^\top\X_S)^{-1}\x_i\x_i^\top)}{s-d}
= \frac{s-\tr((\X_S^\top\X_S)^{-1}\X_S^\top\X_S)}{s-d}
=\frac{s-d}{s-d}=1.$$

It remains to show the formula for the probability $P(S)$ of all paths
ending at node $S$.
If $\det(\X_S^\top\X_S)=0$, then one edge on any path from
the root to $S$ has probability 0. This edge goes from
a superset of $S$ with positive volume to a superset of $S$ that has
volume 0. Since all paths have probability 0, $P(S)=0$ in this case. 

Now assume $\det(\X_S^\top\X_S)>0$ and consider any path from the root
$\{1..n\}$ to $S$. There are $(n-s)!$ such paths all going
through sets with positive volume.
The fractions of determinants in the probabilities along each path telescope 
and the additional factors accumulate to
the same product.
So the probability of all paths from the root to
$S$ is the same and the total probability into $S$ is
$$ \qquad\qquad\frac{(n-s)!}{\Red{(n-d)\ldots(s-d+1)}}
\frac{\det(\X_S^\top\X_S)} {\det(\X^\top\X)}
\;=\ \frac{1}{\Red{n-d \choose s-d}}
\frac{\det(\X_S^\top\X_S)} {\det(\X^\top\X)}.$$
\vspace{-1.35cm}

\end{proof}

An immediate consequence of the above sampling procedure is the
following composition property of volume sampling, which states that this
distribution is closed under subsampling. 
We also give a direct proof to highlight the
combinatorics of volume sampling.
\begin{corollary}
\label{c:comp}
For any $\X\in\R^{n\times d}$ and $n\ge t > s\ge d$, the following
hierarchical sampling procedure:
\begin{align*}
T&\overset{t}{\sim}\X \qquad\ \ \,\text{(size $t$ volume sampling from $\X$)},\\
S&\overset{s}{\sim} \X_T\qquad \text{(size $s$ volume sampling from $\X_T$)}
\end{align*}
returns a set $S$ which is distributed according to size $s$ volume
sampling from $\X$.
\end{corollary}
\begin{proof} We start with the Law of Total Probability
and then use the probability formula for volume sampling
from the above theorem. Here $P(T\cap S)$
means the probability of all paths going through node $T$ at level $t$
and ending up at the final  node $S$ at level $s$. If $S \not\subseteq
T$, then $P(T\cap S)=0$.  
\begin{align*}
P(S)& = \;\sum_{T:\, S\subseteq T}\;\;\;
\overbrace{P(S\,|\,T)\qquad\qquad P(T)}^{P(T\cap S)}\\
&= \;\sum_{T:\,S\subseteq T}
\frac{\det(\X_S^\top\X_S)}{\Red{t-d\choose s-d}
\cancel{\det(\X_T^\top\X_T)}}
\;\;\;
\frac{\cancel{\det(\X_T^\top\X_T)}}{\Red{n-d\choose
t-d}\det(\X^\top\X)}\\
&=\Blue{n-s\choose t-s} \frac{\det(\X_S^\top\X_S)} {\Red{t-d\choose
  s-d}\Red{n-d\choose t-d}\det(\X^\top\X)}\ =\ \frac{\det(\X_S^\top\X_S)} {\Red{n-d\choose
s-d}\det(\X^\top\X)}.
\end{align*}
Note that for all sets $T$ containing $S$, the probability $P(T\cap S)$ is the same,
and there are $\Blue{n-s\choose t-s}$ such sets.
\end{proof}

The main competitor of volume sampling is i.i.d. sampling of
the rows of $\X$ w.r.t. the statistical leverage scores.
For an input matrix $\X\in\R^{n\times d}$, 
the leverage score of the $i$-th row $\x_i^\top$ of $\X$ is defined as
\begin{align*}
l_i \defeq \x_i^\top(\X^\top\X)^{-1}\x_i.
\end{align*}
Recall that this quantity appeared in the definition of conditional
probability $P(\Sm|S)$ in Theorem \ref{t:vol}, where the
leverage score was computed w.r.t. the submatrix $\X_S$. 
In fact, there is a more basic relationship between
leverage scores and volume sampling: If set $S$
is sampled according to size $s=d$ volume sampling, then 
the leverage score $l_i$ of row $i$ is the marginal probability 
$P(i\in S)$ of selecting $i$-th row into $S$.
A general formula for the marginals of size $s$ volume sampling
is given in the following proposition: 
\begin{proposition}\label{prop:marginal}
Let $\X\in\R^{n\times d}$ be a full rank matrix and $s\in\{d..n\}$. If
$S\subseteq \{1..n\}$ is sampled according to size $s$ volume
sampling, then for any $i\in\{1..n\}$,
\vspace{-3mm}
\begin{align*}
P(i\in S) = \frac{s-d}{n-d} + \frac{n-s}{n-d}\,\overbrace{\x_i^\top(\X^\top\X)^{-1}\x_i}^{l_i}.
\end{align*}
\end{proposition}
\begin{proof}
Instead of $P(i\in S)$ we will first compute $P(i\notin S)$:
\begin{align*}
P(i\notin S)&=\sum_{S:|S|=s,i\notin S}  \frac{\det(\X_S^\top\X_S)}
                                  {{n-d\choose s-d} \det(\X^\top\X)}
\\ &=\sum_{S:|S|=s,i\notin S}
\frac{\sum_{T\subseteq S:|T|=d} \det(\X_T^\top\X_T)}
     {{n-d\choose s-d} \det(\X^\top\X)}
\\ &=\frac{\Blue{n-d-1\choose s-d}
        \overbrace{\sum_{T:|T|=d,i\notin T} \det((\X_{-i})_T^\top(\X_{-i})_T)}
                 ^{\det(\X_{-i}^\top\X_{-i})}
          }
          {{n-d\choose s-d} \det(\X^\top\X)}
\\ &=\frac{n-s}{n-d}\; \big(1-\xinv\big),
\end{align*}
where we used Cauchy-Binet twice and the fact
that every set $T:|T|\!=\!d,\ i\notin T$
appears in $\Blue{n-d-1\choose s-d}$
sets $S:|S|\!=\!s,\ i\notin S$.
Now, the marginal probability follows from the fact that $P(i\in S)=1-P(i\notin S)$.
\end{proof}

\subsection{Expectation formulas for volume sampling}
\label{sec:expectation-formulas}
All expectations in the remainder of the paper are w.r.t.
volume sampling. We use the short-hand $\E[\fof{\F}S]$ for
expectation with volume sampling where the size of the
sampled set is fixed to $s$.
The expectation formulas for two choices of $\fof{\F}S$ are
proven in Theorems \ref{t:einv} and \ref{t:einvs}. By Lemma \ref{l:key} 
it suffices to show $\fof{\F}S=\sum_{i\in S} P(\Sm|S)\fof{\F}\Sm$
for volume sampling. We also present a related
expectation formula (Theorem \ref{t:eprojs}), which is proven later
using different techniques.

Recall that $\X_S$ is the submatrix of
rows indexed by $S\subseteq\{1..n\}$.
We also use a version of $\X$ in which
all but the rows of $S$ are zeroed out.
This matrix equals $\I_S\X$ where $\I_S$ is an
$n$-dimensional diagonal
matrix with $(\I_S)_{ii}=1$ if $i\in S$ and 0 otherwise (see Figure \ref{f:shapes}).

\begin{theorem}
\label{t:einv}
Let $\X\in\R^{n\times d}$ be a tall full rank matrix
(i.e. $n\geq d$). For $s\in \{d..n\}$, let
$S\subseteq \{1..n\}$ be a size $s$ volume sampled set over $\X$. Then
$$\E[(\I_S\X)^+]=\X^+.$$
\end{theorem}
For the special case of $s=d$, this fact was known in the linear
algebra literature \citep{bental-teboulle,volume-of-matrices}.
It was shown there using elementary properties of the determinant 
such as Cramer's rule.\footnote{%
Using the composition property of volume sampling (Corollary \ref{c:comp}),
the $s>d$ case of the theorem can be reduced to the $s=d$ case.
However, we give a different self-contained proof.} 
The proof methodology developed here based on reverse
iterative volume sampling is very different.
We believe that this fundamental formula lies at the 
core of why volume sampling is important in many applications. In this
work, we focus on its application to linear regression. However,
\cite{avron-boutsidis13} discuss many problems where controlling the
pseudoinverse of a submatrix is essential. For those
applications, it is important to establish variance bounds for the
above expectation and volume
sampling once again offers very concrete guarantees. We obtain them by
showing the following formula, which can be viewed as a second moment
for this estimator.
\begin{theorem}
\label{t:einvs}
Let $\X\in\R^{n\times d}$ be a full rank matrix and $s\in\{d..n\}$.
If size $s$ volume sampling over $\X$ has full support, then
\[\E[\underbrace{(\X_S^\top\X_S)^{-1}}_
{(\I_S\X)^+(\I_S\X)^{+\top}} ]
= \frac{n-d+1}{s-d+1}\,
\underbrace{(\X^\top\X)^{-1}}_{\X^+\X^{+\top}}.\]
In the case when volume sampling does not have full
support, then the matrix equality ``$=$'' above 
is replaced by the positive-definite inequality ``$\preceq$''.
\end{theorem}
The condition that size $s$ volume sampling over $\X$ has
full support is equivalent to $\det(\X_S^\top\X_S)>0$ for
all $S\subseteq \{1..n\}$ of size $s$.
Note that if size $s$ volume sampling has full support, then size
$t>s$ also has full support. So full support for the
smallest size $d$
(often phrased as $\X$ being {\em in general position})
implies that volume sampling w.r.t. any size $s\ge d$ has full support.

The above theorem immediately gives an expectation formula for
the Frobenius norm $\|(\I_S\X)^+\|_F$ of the estimator:
\begin{align}
\label{e:frobs}
\E\left[\|(\I_S\X)^+\|_F^2\right] &= \E[\tr((\I_S\X)^+
  (\I_S\X)^{+\top})] = \frac{n-d+1}{s-d+1}\|\X^+\|_F^2.
\end{align}
This norm formula was shown by \cite{avron-boutsidis13}, with
numerous applications. Theorem \ref{t:einvs} can be viewed as a
much stronger pre-trace version of the known norm formula. Also our proof
techniques are quite different and much simpler. Note that if size $s$
volume sampling  for $\X$ does not have full support, then
\eqref{e:frobs} becomes an inequality. 

We now mention a second application of the above theorem 
in the context of linear regression 
for the case when the response vector $\y$ is modeled as a noisy linear transformation (i.e., $\y=\X\wbt+\xib$ for some $\wbt\in\R^d$ and
a random noise vector $\xib\in \R^n$ 
(detailed discussion in Section \ref{sec:regularized}). 
In this case the matrix $(\X_S^\top\X_S)^{-1}$ can
be interpreted as the covariance matrix of least-squares estimator
$\of{\w^*}S$ (for a fixed set $S$) and
Theorem \ref{t:einvs} gives an exact formula for the covariance matrix 
of $\of{\w^*}S$ under volume sampling. 
In Section \ref{sec:regularized}, we give an extended version of this
result which provides even stronger guarantees for regularized
least-squares estimators under this model (Theorem \ref{t:reg-einvs}).

Note that except for the above application, all results in
this paper hold for arbitrary response vectors $\y$. 
By combining Theorems \ref{t:einv} and \ref{t:einvs}, we
can also obtain a covariance-type 
formula\footnote{This notion of ``covariance'' is
  used in random matrix theory, i.e. for a random matrix $\M$ 
  we analyze $\E[(\M-\E[\M])(\M-\E[\M])^\top]$. See for example
  \cite{matrix-tail-bounds}.} for the pseudoinverse matrix estimator:
\begin{align}
&\E[((\I_S\X)^+-\E[(\I_S\X)^+])\; ((\I_S\X)^+-\E[(\I_S\X)^+])^\top]  
\nonumber\\ 
&=\E[(\I_S\X)^+(\I_S\X)^{+\top}]  - \E[(\I_S\X)^{+}]\; \E[(\I_S\X)^+]^\top
\nonumber\\ 
&=\frac{n-d+1}{s-d+1} \;\X^{+}\X^{+\top} - \X^{+}\X^{+\top}
  =\frac{n-s}{s-d+1}\; \X^{+}\X^{+\top}.
\label{e:covs}
\end{align}

We now give the background for a third matrix expectation formula
for volume sampling. Pseudoinverses can be used to compute the 
projection matrix onto the span of columns of matrix $\X$, 
which is defined as follows:
\vspace{-3mm}
\begin{align*}
\P_\X \defeq \X\overbrace{(\X^\top\X)^{-1}\X^\top}^{\X^+}.
\end{align*}
Applying Theorem \ref{t:einv} leads us immediately to the following unbiased
matrix estimator for the projection matrix:
\begin{align*}
\E[\X(\I_S\X)^+] = \X\,\E[(\I_S\X)^+] = \X\X^+ = \P_\X.
\end{align*}
Note that this matrix estimator $\X(\I_S\X)^+$
is closely connected to linear regression: 
It can be used to transform the response vector $\y$
into the prediction vector $\ybh(S)$ of subsampled least squares
solution $\of{\w^*}S$ as follows: 
$$\ybh(S) = \X\underbrace{(\I_S\X)^+\y}_{\w^*(S)}.$$
In this case, volume sampling once again provides a covariance-type 
matrix expectation formula. 

\begin{theorem}
\label{t:eprojs}
Let $\X\in\R^{n\times d}$ be a full rank matrix. If matrix $\X$ is in
general position and $S\subseteq\{1..n\}$ is sampled according to size
$d$ volume sampling, then
\begin{align*}
\E[\!\!\!\!\!\underbrace{(\X(\I_S\X)^{+})^2}_{(\I_S\X)^{+\top}\X^\top\X(\I_S\X)^+}\!\!\!\!]
  - \P_\X = d\, (\I - \P_\X).
\end{align*}
If $\X$ is not in general position, then the matrix equality ``$=$''
is replaced by the  positive-definite inequality ``$\preceq$''.
\end{theorem}
Note that this third expectation formula is limited to sample size $s=d$.
It is a direct consequence of Theorem \ref{t:loss} given in the next section
which relates the expected loss of a subsampled least squares estimator
to the loss of the optimum least squares estimator. 
Unlike the first two formulas given in theorems \ref{t:einv} and \ref{t:einvs},
its proof does not rely on the methodology of Lemma \ref{l:key}, i.e., on 
showing that the expectations at all levels of a certain
DAG associated with the sampling process are the same.
We defer the proof of this third expectation formula to 
the end of Section \ref{sec:proof-loss}. No extension of this
third formula to sample size $s>d$ is known. 

\begin{proofof}{Theorem}{\ref{t:einv}}
We apply Lemma \ref{l:key} with $\fof{\F}S= (\I_S\X)^+$.
It suffices to show $\fof{\F}S=\sum_{i\in S} P(\Sm|S)\fof{\F}\Sm$ 
for $P(\Sm|S)=\frac{1-\x_i^\top(\X_S^\top\X_S)^{-1}\x_i}{s-d}$, i.e.:
$$(\I_S\X)^+
= \sum_{i\in S} \frac{1-\x_i^\top(\X_S^\top\X_S)^{-1}\x_i}{s-d}\!\!
\underbrace{(\I_\Sm\X)^+}
_{(\X_\Sm^\top\X_\Sm)^{-1}(\I_\Sm\X)^\top }\hspace{-0.9cm}.$$
We first apply Sherman-Morrison to
$(\X_\Sm^\top\X_\Sm)^{-1}=
(\X_S^\top\X_S-\x_i\x_i^\top)^{-1}$ on the r.h.s. of the above:
$$\sum_i \frac{1-\x_i^\top(\X_S^\top\X_S)^{-1}\x_i}  
               {s-d} \quad
\left((\X_S^\top\X_S)^{-1} +\frac{(\X_S^\top\X_S)^{-1} \x_i\x_i^\top(\X_S^\top\X_S)^{-1}} {1-\x_i^\top(\X_S^\top\X_S)^{-1}\x_i}
	      \right)             ((\I_S\X)^\top-\x_i\e_i^\top)
.$$
Next we expand the last two factors into 4 terms.
The expectation of the first
$(\X_S^\top\X_S)^{-1}(\I_S\X)^\top$ is $(\I_S\X)^+$
(which is the l.h.s.) and the expectations of the remaining three terms times $s-d$ 
sum to 0:
\begin{align*}
&-\sum_{i\in S} (1-\x_i^\top(\X_S^\top\X_S)^{-1}\x_i)\, (\X_S^\top\X_S)^{-1}\x_i\e_i^\top
+\cancel{(\X_S^\top\X_S)^{-1}} \cancel{\sum_{i\in S} \x_i\x_i^\top} (\X_S^\top\X_S)^{-1} (\I_S\X)^\top
\\&\qquad 
-\sum_{i\in S}(\X_S^\top\X_S)^{-1} \x_i\;(\x_i^\top(\X_S^\top\X_S)^{-1}\x_i)\e_i^\top
= 0.
\end{align*}
\vspace{-1.3cm}

\end{proofof}\\
In Appendix \ref{sec:alternate-proof} 
we give an alternate proof using a derivative argument.

\begin{proofof}{Theorem}{\ref{t:einvs}}
Choose $\fof{\F}S= \frac{s-d+1}{n-d+1} (\X_S^\top\X_S)^{-1}$.
By Lemma \ref{l:key} it suffices to 
show $\fof{\F}S=\sum_{i\in S} P(\Sm|S)\fof{\F}\Sm$ for volume sampling:
$$\frac{s-d+1}{\cancel{n-d+1}} (\X_S^\top\X_S)^{-1}
= \sum_{i\in S} \frac{1-\x_i^\top(\X_S^\top\X_S)^{-1}\x_i}{\cancel{s-d}}
  \frac{\cancel{s-d}}{\cancel{n-d+1}} (\X_\Sm^\top\X_\Sm)^{-1}.
$$
To show this we apply Sherman-Morrison to 
$(\X_\Sm^\top\X_\Sm)^{-1}$ on the r.h.s.: 
\begin{align*}
&\sum_{i\in S} (1-\x_i^\top(\X_S^\top\X_S)^{-1}\x_i)
\left((\X_S^\top\X_S)^{-1} +\frac{(\X_S^\top\X_S)^{-1} \x_i\x_i^\top(\X_S^\top\X_S)^{-1}}
{1-\x_i^\top(\X_S^\top\X_S)^{-1}\x_i}\right)
\\&\Blue{=} \,(s-d) (\X_S^\top\X_S)^{-1}
       +  \cancel{(\X_S^\top\X_S)^{-1}} \cancel{\sum_{i\in S}
\x_i\x_i^\top} (\X_S^\top\X_S)^{-1}
=(s-d+1)\;(\X_S^\top\X_S)^{-1}. 
\end{align*}
If some denominators $1-\x_i^\top(\X_S^\top\X_S)^{-1}\x_i$ are zero, then we
only sum over $i$ for which the denominators are
positive. In this case the above matrix equality becomes a
positive-definite inequality $\Blue{\preceq}$.
\end{proofof}

\section{Linear regression with smallest number of responses}
\label{sec:linear-regression}
Our main motivation for studying volume sampling came from asking
the following simple question. Suppose we want
to solve a $d$-dimensional linear regression problem with
an input matrix $\X$ of $n$ rows in $\R^d$ and a
response vector $\y\in\R^n$, i.e. find
$\w\in\R^d$ that minimizes the least squares loss $\|\X\w-\y\|^2$ 
on all $n$ rows. We use $L(\w)$ to denote this
loss. The optimal weight vector minimizes $L(\w)$, i.e.
\vspace{-2mm}
\[\w^*\defeq \argmin_{\w\in\R^d}L(\w) =\X^+\y.\]
Computing it requires access to the input matrix $\X$ and
the response vector $\y$. Assume we are given $\X$
but the access to response vector $\y$ is restricted. 
We are allowed to pick a random subset $S\subseteq\{1..n\}$ 
of fixed size $s$ for which the responses $\y_S$ for the submatrix $\X_S$ are
revealed to us, and then must produce a weight vector
$\w(\X,S,\y_S)\in \R^d$ from a subset of row indices $S$
of the input matrix $\X$ and the corresponding responses $\y_S$.
Our goal in this paper 
is to find a distribution on the subsets $S$ of size $s$
and a {\em weight function} $\w(\X,S,\y_S)$ s.t.%
\footnote{Since the learner is given $\X$, it is natural
to define the optimal multiplicative constant specialized
for each $\X$:
$c_{\X,s}=\min_c\min_{P(\cdot),\w(\cdot)}\max_\y\, 
\E_P\,[L(\w(\X,S,\y_S))] \le  (1+c)\, L(\w^*)$,
where the domain for distribution $P(\cdot)$ and weight 
function $\w(\cdot)$ are sets of size $s$. Showing specialized bounds
for $c_{\X,s}$ is left for future research.
}
$$\forall\, (\X,\y)\in \R^{n\times d}\times \R^{n\times
1}:\;\; \E\,[L(\w(\X,S,\y_S))] \le\Magenta{ (1+c)}\, L(\w^*),
$$
where $c$ must be a fixed constant (that is independent of $\X$ and
$\y$).
Throughout the paper we use the one argument shorthand $\w(S)$ 
for the weight function $\w(\X,S,\y_S)$.
We assume that attaining response values is expensive and ask the question:
What is the smallest number of responses (i.e. smallest size of $S$) 
for which such a multiplicative bound is possible?
We will use volume sampling to show that attaining $d$
response values is sufficient and show that less than $d$
responses is not.

\begin{wrapfigure}{r}{0.45\textwidth}
\vspace{-4.5mm}
\begin{tikzpicture}[font=\small,scale=1,pin distance=1.6mm]
    \begin{axis}[hide axis, xmin=-2.35,xmax=1.35,ymin=-2.2,ymax = 5.1]
        \addplot [domain=-2.07:1.1,samples=250, ultra thick, blue] 
	{x^2} node [pos=0.3, xshift=-.5cm] {$L(\cdot)$};
        \addplot [domain=-2.1:1.1,samples=250, ultra thick, red ] {2-x};
	\addplot[mark=none, ultra thick, green] coordinates {(-2.15,-1) (1.15,-1)};
        \addplot[mark=*] coordinates {(0,0)} node[pin=-20:{$L(\w^*)$}]{};
        \addplot[mark=*] coordinates {(0,2)} node[pin=90:{$\E[L(\of{\w^*}S)]$}]{};
        \addplot[mark=*] coordinates {(-2,4)} node[pin=90:{$\,\,\,L(\of{\w^*}{S_i})$}]{};
        \addplot[mark=*] coordinates {(1,1)} node[pin=90:{$L(\of{\w^*}{S_j})\,\,\,$}]{};
        \addplot[mark=*] coordinates {(-2,-1)} node[pin=-90:{$\of{\w^*}{S_i}$}]{};
        \addplot[mark=*] coordinates {(1,-1)} node[pin=-90:{$\of{\w^*}{S_j}$}]{};
        \addplot[mark=*] coordinates {(0,-1)} node[pin=-90:{$\w^*=\E[\of{\w^*}S]$}]{};
	\draw [decorate,decoration={brace,amplitude=4.5pt},xshift=-2.5pt,yshift=0pt]
	(0,0) -- (0,2) node [black,midway,xshift=-.8cm] {$d\,L(\w^*\!)$};
    \end{axis}
    \end{tikzpicture}
\vspace{-6.9mm}
\caption{Unbiased estimator $\of{\w^*}S$ in expectation suffers loss
  $(d+1)\,L(\w^*)$.}
\label{fig:jensen}
\vspace{-3mm}
\end{wrapfigure}
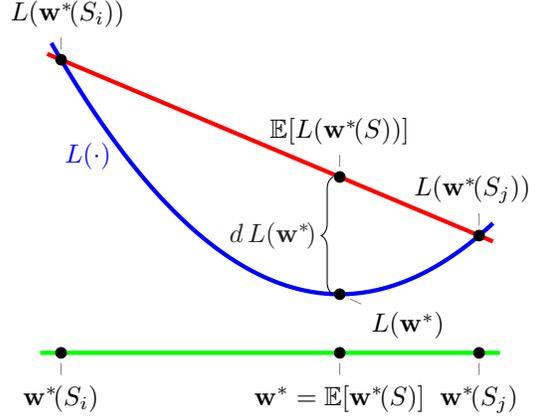

Before we state our main upper bound based on volume
sampling, we make the following key observation:
If for the subproblem $(\X_S,\y_S)$ there is a weight
vector $\w(S)$ that has loss zero, then the algorithm has
to predict with such a consistent weight vector. This is because in
that case the responses $\y_S$ can be extended to a
response vector $\y$ for all of $\X$ s.t. $L(\w^*)=0$. 
Thus since we aim for a
multiplicative loss bound, we force the algorithm to
predict with the optimum solution $\w^*\!(S)\defeq(\X_S)^+\y_S$
whenever the subproblem $(\X_S,\y_S)$ has loss 0.
In particular, when $|S|= d$ and $\X_S$ has full rank, 
then there is a unique consistent solution $\w^*\!(S)$ for the subproblem 
and the learner must use the weight function $\w(S)=\w^*\!(S)$.

\begin{theorem}
\label{t:loss}
If the input matrix $\X\in\R^{n\times d}$ is in general position, 
then for any response vector $\y\in \R^n$, the expected
square loss (on all $n$ rows of $\X$) of the optimal solution
$\of{\w^*}S$ for the subproblem 
$(\X_S,\y_S)$, with the $d$-element set $S$ obtained from 
volume sampling, is given by
\begin{align*}
\E[L(\of{\w^*}S)] =(d+1)\; L(\w^*).
\end{align*}
If $\X$ is not in general position, then the expected loss is
upper-bounded by $(d+1)\; L(\w^*)$.
\end{theorem}
There are no range restrictions on the $n$ points and
response values in this bound.
Also, as discussed in the introduction, 
this bound is already non-obvious for dimension 1, when
the multiplicative factor is 2 (See Figure \ref{f:one} for a visualization). 
Note that if there is a bias term in dimension 1, then the factor becomes 3. 

In dimension $d$, it is instructive to look at the case when
the square loss of the optimum solution is zero, 
i.e. there is a weight vector $\w^*\in \R^d$ s.t.  $\X\w^*=\y$. 
In this case the response values of any $d$ linearly independent rows
of $\X$ determine the optimum solution and 
the multiplicative loss formula of the theorem clearly holds.
The formula specifies how noise-free case generalizes
gracefully to the noisy case in that for volume sampling,
the expected square loss of the solution obtained from $d$ row 
response pairs is always by a factor of at most $d+1$ larger than 
the square loss of the optimum solution.
Moreover, since $\E[\of{\w^*}S]=\w^*$ and the loss function $L(\cdot)$ is
convex, we have by Jensen's inequality that
$$\E\big[L(\of{\w^*}S)\big]\ge L\big(\E[\of{\w^*}S]\big)=L(\w^*).$$ 
The above theorem now states that the gap $\E[L(\of{\w^*}S)] - L(\w^*)$
in Jensen's inequality (which coincides with the ``regret'' of the
estimator) equals $d\, L(\w^*)$, when the
expectation is w.r.t. size $d$ volume sampling and $\X$ is in
general position (See Figure \ref{fig:jensen} for a schematic).
As we will show in Section \ref{sec:av}, 
this gap also equals the variance $\E[\|\X\of{\w^*}S
-\X\w^*\|^2]$ of the predictions since the estimator is
unbiased. In summary:
\vspace{-3mm}
$$\underbrace{\E\big[L(\of{\w^*}S)\big] -
L(\!\!\overbrace{\w^*}^{\E[\w^*\!(S)]}\!\!)}_{\text{regret}}
= \!\underbrace{d\, L(\w^*)}_{\text{gap in Jensen's}} \!
= \underbrace{\E\big[\|\X\of{\w^*}S
-\X\w^*\|^2\big]}_{\text{variance}}.$$

We now make a number of
observations and present some lower bounds that highlight
the upper bound of the above theorem. Then, in Section \ref{sec:proof-loss} we prove the
theorem and a matrix expectation formula implied by it.

\subsection*{When $\X$ is not in general position}
\label{sec:partially-degenerate}

The above theorem gives an equality 
for the expected loss of a volume-sampled solution.
However, this equality is only
guaranteed to hold when matrix $\X$ is in general position.
We give a minimal example problem where the matrix
$\X$ is not in general position and the equality
of Theorem \ref{t:loss} turns into a strict inequality.
This shows that for the equality, the general position assumption is necessary.
If we apply even an infinitesimal additive perturbation to
the matrix $\X$ of the example problem, 
then the resulting matrix $\X_{\epsilon}$ is in general position
and the equality holds.
Note that even though the optimum loss
$L(\w^*)$ does not change significantly under such a
perturbation, the expected sampling loss $\E[L(\of{\w^*}S)]$ has to
jump sufficiently to close the gap in the inequality.
In our minimal example problem, $n=3$ and $d=2$, and
\[
\X = \begin{pmatrix}
1 & 1  \\
1 & 1  \\
1 & 0  \end{pmatrix},\quad
\y = \begin{pmatrix}
1 \\
0 \\
0           \end{pmatrix}.
\]
We have three 2-element subsets to sample from:
$S_1=\{1,2\},\ S_2=\{2,3\},\ S_3=\{1,3\}.$
Notice that the first two rows of $\X$ are identical, which
means that the probability of sampling set $S_1$ is 0 in
the volume sampling process. The
other two subsets, $S_2$ and $S_3$, form identical submatrices
$\X_{S_2}=\X_{S_3}$. Therefore they are equally probable.
The optimal weight vectors for these sets are
$\of{\w^*}{S_2}=(0,0)^\top$ and  $\of{\w^*}{S_3} = (0,1)^\top$.
Also $\w^*=(0,\frac{1}{2})^\top$ and the expected loss is
bounded as:
\begin{align*}
\E[L(\of{\w^*}S)]  
=  \underbrace{%
   \frac{1}{2} \overbrace{L(\of{\w^*}{S_2})}^1 
 + \frac{1}{2} \overbrace{L(\of{\w^*}{S_3})}^1
              }_1
\;\;<\;\;
   \underbrace{\overbrace{(d+1)}^3\,\overbrace{L(\w^*)}^{1/2}}_{3/2}.
\end{align*}
Now consider a slightly perturbed input matrix 
\[
\X_\epsilon = \begin{pmatrix}
1 & 1+\epsilon  \\
1 & 1  \\
1 & 0  \end{pmatrix},
\]
where $\epsilon > 0$ is arbitrarily small 
(We keep the response vector $\y$ the same). 
Now, there is no $d\times d$ submatrix that is singular, so
the upper bound from Theorem \ref{t:loss} must be tight. 
The reason is that even though subset $S_1$ still
has very small probability, its loss is
very large,
so the expectation is significantly affected by this component, no matter
how small $\epsilon$ is. We see this directly in the
calculations. Let $\w^*$ and $\of{\w^*}{S_i}$ be the corresponding
solutions for the perturbed problem and its subproblems.
The volumes of the subproblems and their losses are:
\begin{align*}
\begin{array}{ll}
\det(\X_{S_1}^\top\X_{S_1}) = \epsilon^2 & L(\of{\w^*}{S_1})=\epsilon^{-2}\\
\det(\X_{S_2}^\top\X_{S_2}) = 1 & L(\of{\w^*}{S_2})=1\\
\det(\X_{S_3}^\top\X_{S_3}) = (1+\epsilon)^2 & L(\of{\w^*}{S_3})=(1+\epsilon)^{-2}
\end{array}
\qquad L(\w^*) = \frac{1}{2(1+\epsilon+\epsilon^2)}.
\end{align*}
Note that for each subproblem, the product of volume times loss is
equal to 1. Now the expected loss can be easily computed, and we can
see that the gap in the bound disappears (the denominator 
is the normalizing constant for volume sampling):
\begin{align*}
\E[L(\of{\w^*}S)]  &=
\frac{1+1+1}{\epsilon^2 + 1 + (1+\epsilon)^2} =
(d+1)\; L(\w^*).  
\end{align*}

\subsection{Lower-bounds} 

The factor $d+1$ in Theorem \ref{t:loss} cannot, in
general, be improved when selecting only $d$ responses: 
\begin{proposition}
\label{prop:optimal}
For any $d$, there exists a least squares problem $(\X,\y)$ with $d+1$
rows in $\R^d$ such that for every $d$-element index set
$S\subseteq\{1\,..\,d+\!1\}$, we have \[L(\of{\w^*}S) = (d+1)\;L(\w^*).\]
\end{proposition}
\begin{proof}
Choose the input vectors $\x_i$ (and rows $\x_i^\top$) 
as the $d+1$ corners of the
simplex in $\R^d$ centered at the origin and
choose all $d+1$ responses as the same non-zero value $\alpha$.
For any $\alpha$, the optimal solution $\w^*$
will be the all-zeros vector with loss
\[L(\w^*) = (d+1)\;\alpha^2. \]
On the other hand, taking any size $d$ subset of indices
$S\subseteq\{1\,..\,d+\!1\}$, the subproblem solution
$\of{\w^*}S$ will only produce loss on the left out input
vector $\x_i$, indexed with $i\not\in S$.
To obtain the prediction on $x_i$, we use a simple geometric
argument. Observe that since the simplex is centered, we can write the
origin of $\R^d$ in terms of the corners of the simplex as
\begin{align*}
\zero = \sum_k\x_k = \x_i +
  d\,\bar{\x}_{-i},\quad\text{ where } \bar{\x}_{-i}\defeq\frac{1}{d}\sum_{k\neq i}\x_k.
\end{align*}
Thus, the left out input vector $\x_i$ equals $-d\,\bar{\x}_{-i}$. 
The prediction of $\of{\w^*}S$ on this vector is
\begin{align*}
\yh_{i} 
= \x_{i}^\top\of{\w^*}S 
= -d \bigg(\frac{1}{d} \sum_{k\neq i}\x_k^\top\bigg)\of{\w^*}S 
= -\sum_{k\neq i}\x_k^\top\of{\w^*}S
= -d\alpha.
\end{align*}
It follows that the loss of $\of{\w^*}S$ equals
\begin{align*}
 L( \of{\w^*}S ) &=(\yh_{i} - y_{i})^2 =
(-d\alpha - \alpha)^2 = (d+1)^2\alpha^2 = (d+1)\,L(\w^*).
\\[-1.3cm]
\end{align*}
\end{proof}

Moreover, it is easy to show that no {\em deterministic}
algorithm for selecting $d$ rows (without knowing the responses) 
can guarantee a multiplicative loss bound with a factor less than $n/d$
\citep{coresets-regression}. For the sake of completeness, we show this
here for $d=1$:
\begin{proposition}
\label{prop:lb-deterministic}
For any $n\times 1$ input matrix $\X$ of all 1's and
any deterministic algorithm that chooses some singleton set $S=\{i\}$,
there is a response vector $\y$
for which the loss of the subproblem 
and the optimal loss are related as follows:
\[L(\of{\w^*}S)=n\, L(\w^*).\]
\end{proposition}

\noindent {\bf Proof} \ 
If the response vector $\y$ is the vector of $n$ 1's 
except for a single 0 at index $i$, then we have
\begin{align*}
\underbrace{L(\overbrace{\of{\w^*}{\{i\}}}^0}_{n-1})
= n\, \underbrace{L(\overbrace{\w^*}^{\frac{n-1}{n}})}_{\frac{n-1}{n}}.
\\[-1.9cm]
\end{align*}
\hfill\BlackBox

\vspace{1cm}
Note that for the 1-dimensional example used in the proof, volume sampling
would pick the set $S$ uniformly. For this distribution, 
the multiplicative factor drops from $n$ downto 2, that is
$\E[L(\of{\w^*}S)] =\frac{1}{n} (n-1)+\frac{n-1}{n} 1= 2\; L(\w^*).$

\subsection*{The importance of joint sampling}

Three properties of volume sampling play a crucial role in achieving a
multiplicative loss bound:
\begin{enumerate}
\item \textit{Randomness}: No deterministic algorithm
 guarantees such a bound (see Proposition \ref{prop:lb-deterministic}).
\item \textit{The chosen submatrices must have full rank}:
Choosing any rank deficient submatrix with positive
probability, does not allow for a multiplicative bound
(see Propositions \ref{prop:rankdef} and \ref{prop:d-1}).
\item \textit{Jointness}: No i.i.d. sampling procedure
can achieve a multiplicative loss bound with $O(d)$
responses (see Corollary \ref{cor:lb-iid}).
\end{enumerate}

By jointly selecting subset $S$, volume sampling ensures that the
corresponding input vectors $\x_i$ are well spread out in the input space
$\R^d$. In particular, volume sampling does not
put any probability mass on sets $S$ such that the rank of submatrix
$\X_S$ is less than $d$. Intuitively, selecting rank deficient
row subsets should not be effective, since such a
choice leads to an under-determined least squares problem. We make
this simple statement more precise by showing that any randomized
algorithm, that with positive probability selects a rank deficient
row subset, cannot achieve a multiplicative loss bound.
Intuitively if the algorithm picks a rank deficient subset
then it is not clear how it should select the
weight vector $\w(S)$ given input matrix $\X$, subset $S$
and responses $\y_S$. 
We reasoned before that $\w(S)$ must have loss 0 on the
subproblem $(\X_S,\y_S)$. However if $\rank(\X_S)<d$, then
the choice of the weight vector $\w(S)$ with loss 0 is not unique
and this causes positive loss for some response vector $\y$.

\begin{proposition}
\label{prop:rankdef}
If for any input matrix $\X$, the algorithm
samples a rank deficient subset $S$ of rows with positive
probability, then the expected loss of the algorithm cannot
be bounded by a constant times the optimum loss for all
response vectors $\y$.
\end{proposition}
Note that this means in particular that if $\X$ has rank $d$, then
sampling $d-1$ size subsets with positive probability does
not allow for a constant factor approximation.\\[-2mm]
\begin{proof}
Let $S$ be a rank deficient subset chosen with probability $P(S)>0$.
Since in our setup the bound has to hold for all response
vectors $\y$ we can imagine an adversary choosing a worst-case $\y$. 
This adversary gives all rows of $\X_S$ the response value zero.
Let $\w(S)$ be the plane produced by the algorithm
when choosing $S$ and receiving the responses 0 for $\X_S$.
Let $i\in\{1..n\}$ s.t. $\x_i^\top\not\in\text{row-span}(\X_{S})$
and let $\w^*$ be any weight vector that gives response
value 0 to all rows of $\X_S$ and response value 
$\x_i^\top \w(S) + Y$ to $\x_i$.
The adversary chooses $\y$ as $\X\w^*$, i.e. it gives all points $\x_j$ 
not indexed by $S$ and different from $\x_i$ the response
values $\x_j^\top\w^*$ as well.
Now $\w^*$ has total loss 0 but $\w(S)$ has loss $Y^2$ on
$\x_i$ and the algorithm's expected total loss is $\ge P(S)\,Y^2$.
\end{proof}

We now strengthen the above proposition in that whenever
the sample $S$ is rank deficient then the loss of
the optimum is zero while the loss of the algorithm is
positive. However note that this proposition is weaker than
the above in that it only holds for specific input matrices.

\begin{proposition}
\label{prop:d-1}
Let $d\leq n$ and let $\X$ be any input matrix 
of rank $d$ consisting of $n$ standard basis row vectors in $\R^d$.
Then for any randomized learning
algorithm that with probability $p$ selects a subset $S$ s.t. $\rank(\X_S)<d$
and any weight function $\w(\cdot)$, there is a response
vector $\y$, satisfying: 
\[L(\w^*) = 0,
\quad \text{ and }
\quad L(\w(S)) > 0
\quad\text{ with probability at least }p.\]
\end{proposition}

\begin{proof}
Let $Q=\{1,2,\ldots,2^n\}$. The adversarial response vector $\y$ is constructed 
by carefully selecting one of the weight vectors $\w^*\in Q^d$, 
and setting the response vector $\y$ to $\X\w^*$. This ensures that
$L(\w^*)=0$ and since $\X$ consists of standard basis row
vectors, the components of $\y$ lie in $Q$ as well. Note that if the learner does
not discover $\w^*$ exactly, it will incur positive loss.
Let $\Hc$ be the set of all rank deficient sets in $\X$, i.e.
those that lack at least one of the standard basis vectors:
\begin{align*}
\Hc = \{S\subseteq \{1..n\}\ :\ \rank(\X_S) < d\}.
\end{align*}
Suppose that given matrix $\X$, the learner uses weight function $\w(S,\y_S)$. 
(Note that for the sake of concreteness we stopped using the single argument
shorthand for the weight function during this proof.)
We will count the number of possible inputs
to this function, when $S$ is a rank deficient index set
of the rows of $\X$ and the response vector $\y_S$ is consistent with
some $\w^*\in Q^d$. For any fixed rank deficient set $S$,
let $t$ be the number of distinct basis vectors appearing
in $\X_S$. Clearly $t \le d-1$. Fix a
subset $T\subseteq S$ of size $t$ s.t. $\X_T$ contains all $t$ basis vectors
of $\X_S$ exactly once (Thus the basis vectors in $\X_{S\setminus T}$
are all duplicates). 
Since $\y\in Q^n$, the components of $\y_S$ also lie in $Q$ 
and $\y_S$ is determined by the responses of $\y_T$. Clearly there
are at most $|Q|^{d-1}$ choices for $\y_T$.
It follows that the number of possible
input pairs $(S,\y_S)$ for function $\w(\cdot,\cdot)$ under the above
restrictions can be bounded as 
 \begin{align*}
 \left|\left\{(S,\y_S)\ :\ [S\in \Hc] \text{ and }
   [\y_S=\X_S\w^*\text{ for } \w^*\in Q^d]\right\}\right|&\leq  
 \underbrace{|\Hc|}_{< 2^n}\ \ 
 \underbrace{\max_{S\in \Hc} |\{\X_S\w^*\ :\ \w^*\in Q^d\} |}_{\leq |Q|^{d-1}}\\
&< 2^n |Q|^{d-1} = |Q^d|.
  \end{align*}
So for every weight function $\w(\cdot,\cdot)$, 
there exists $\w^*\in Q^d$ that is not present in the set
$\{\w(S,\y_S):S\in\Hc\}$. Selecting $\y=\X\w^*$
for the adversarial response vector, we guarantee that the learner
picks the wrong solution for every rank deficient set $S$
and therefore receives positive loss w.p. at least $p$.
\end{proof}

Using Proposition \ref{prop:d-1}, we show that any i.i.d. row sampling
distribution (like for example leverage score sampling) requires
$\Omega(d\log d)$  samples
to get any multiplicative loss bound, either with high probability or
in expectation. 

\begin{corollary}
\label{cor:lb-iid}
Let $d\leq n$ and let $\X$ be any input matrix
of rank $d$ consisting of $n$ standard basis row vectors in $\R^d$.
Then for any randomized learning algorithm which 
selects a random multiset $S\subseteq\{1..n\}$ of size $|S|\leq
  (d-1)\ln(d)$ via i.i.d. sampling from any distribution
and uses any weight function $\w(S)$, there is a response
vector $\y$ satisfying: 
\[L(\w^*) = 0,\quad \text{ and }\quad L(\w(S)) > 0\quad\text{ with
    probability at least }1/2. \] 
\end{corollary}
\begin{proof}
Any i.i.d. sample of size at most $(d-1)\ln(d)$ with probability at
least $1/2$ does not contain all of the unique standard basis vectors (Coupon
Collector Problem\footnote{This was proven for uniform sampling in Theorem 1.24 of
\cite{randomized-heuristics}. It can be shown that uniform sampling is the
best case for Coupon Collector Problem
\citep{coupon-collector-extreme}, so the bound holds for
any i.i.d. sampling.}). Thus, with  probability at least $1/2$ submatrix
$\X_S$ has rank less than $d$. Now, for any such algorithm we can use
Proposition \ref{prop:d-1} to select a consistent adversarial
response vector $\y$ such that with probability at least $1/2$ the
loss $L(\w(S))$ is positive.
\end{proof}
Note that the corollary requires $\X$ to be of a
restricted form that contains a lot of duplicate rows. 
It is open whether this corollary still
holds when $\X$ is an arbitrary full rank matrix.

\subsection{Loss expectation formula (proof of Theorem \ref{t:loss})}
\label{sec:proof-loss}

First, we discuss several key connections between linear regression
and volume, which are used in the proof. Note that the loss $L(\w^*)$
suffered by the optimum weight vector can be written as
$\|\ybh-\y\|^2$, the squared Euclidean distance between  
prediction vector $\ybh=\X\w^*$ and the response vector $\y$. 
Since $\ybh$ is minimizing the distance from $\y$ to the subspace of $\R^n$
spanning the feature vectors $\{\f_1,\dots,\f_d\}$ (columns of $\X$),
it has to be the {\em projection} of $\y$ onto that subspace (see Figure 
\ref{fig:projection}). We denote this projection as
$\P_\X\,\y$, as defined in Section \ref{sec:expectation-formulas}. Note
that $\P_\X$ is a linear mapping
from $\R^n$ onto the column span of the matrix $\X$ such that
\begin{align}
\text{for }\u \in \Span(\X)\quad 
\u = \P_\X\,\y \ \  \Leftrightarrow \ \  \P_\X\,(\u-\y)=\zero
\ \ \Leftrightarrow\ \X^\top(\u-\y)=\zero.
\label{eq:proj}
\end{align}

 We next give a second geometric interpretation of
the length $\|\ybh-\y\|^2$.
Let $\Pc$ be the parallelepiped formed by the $d$
column/feature vectors of the input matrix $\X$.
Furthermore, consider the extended input matrix 
produced by adding the response vector $\y$ to $\X$ as an extra column:
\begin{align}
\Xs\defeq (\X,\y)
  \in\R^{n\times (d+1)}.
\label{eq:Xs}
\end{align}
\begin{wrapfigure}{r}{0.45\textwidth}
\vspace{-38 pt}
  \begin{center}
    \includegraphics[width=0.45\textwidth]{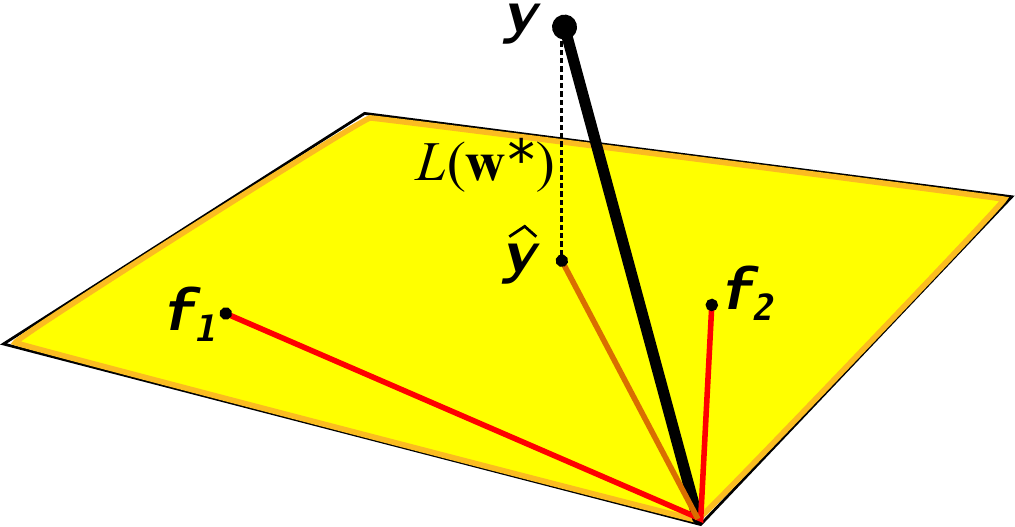}
  \end{center}
\vspace{-15pt}
  \caption{Prediction vector $\ybh$ is a projection of $\y$ onto the
    span of feature vectors $\f_i$.}
\label{fig:projection}
\end{wrapfigure}
Using the ``base $\times$ height'' formula we can relate the volume of
$\Pc$ to the volume of $\widetilde{\Pc}$, the parallelepiped
formed by the $d+1$ columns of $\Xs$. 
Observe that $\widetilde{\Pc}$ has $\Pc$ as one
of its faces, with the response vector $\y$ representing the edge that
protrudes from that face. Hence the volume of $\widetilde{\Pc}$ is the
product of the volume of $\Pc$ and 
the distance between $\y$ and $\text{span}(\X)$. This
distance equals $\|\ybh-\y\|$, since as discussed above, $\ybh$ is the projection of $\y$ onto
$\text{span}(\X)$. Thus we have
\begin{align}
\det(\Xs^\top\Xs)
=\det(\X^\top\X) \,L(\w^*).
\label{eq:base-x-height}
\end{align}

Next, we present a proposition whose corollary is key to proving
Theorem \ref{t:loss}.
Suppose that we select one test row from the input matrix
and use the remaining $n-1$ row response pairs
as the training set. The proposition
relates the loss of the obtained solution on the test row 
to the total leave-one-out loss an all rows.
\begin{proposition}
\label{prop:geometry}
For any index $i\in\n$, let $\of{\w^*}{-i}$ be the solution to the
reduced linear regression problem $(\X_{-i},\y_{-i})$. Then
\vspace{-5mm}
\begin{align*}
L(\of{\w^*}{-i})-L(\w^*) =
\overbrace{\x_i^\top(\X^\top\X)^{-1}\x_i}^{\frac{\det(\X^\top\X)-\det(\X_{-i}^\top\X_{-i})}{\det(\X^\top\X)}} \ell_i(\of{\w^*}{-i}),
\end{align*}
where $\ell_i(\w) \defeq (\x_i^\top\w - y_i)^2$ is the square loss of
$\w$ on the $i$-th point.
\end{proposition}
An algebraic proof of this proposition essentially appears
in the proof of Theorem 11.7 in \cite{prediction-learning-games}. For
the sake of completeness we give a new geometric proof of this proposition 
in Appendix \ref{a:prop-proof} using basic properties of volume, 
thus stressing the connection to volume sampling. 

Note that if matrix $\X$ has exactly $n=d+1$ rows and the
training matrix
$\X_{-i}$ is full rank, then $\of{\w^*}{-i}$ has loss zero on all training
rows. In this case we obtain a simpler relationship than 
the proposition.
\begin{corollary}
\label{cor:geometry}
If $\X$ has $d+1$ rows and $\rank(\X_{-i})=d$, then defining $\Xs$
as in (\ref{eq:Xs}), we have 
\[ \det(\Xs^\top\Xs) = \det(\X_{-i}^\top\X_{-i}) \;\ell_i(\of{\w^*}{-i}).\] 
\end{corollary}
\begin{proof}
By Proposition \ref{prop:geometry} and the fact that
$L(\of{\w^*}{-i})=\ell_i(\of{\w^*}{-i})$, we have
$$\det(\X^\top\X)\; L(\w^*)= \det(\X_{-i}^\top\X_{-i}) \;\ell_i(\of{\w^*}{-i}). 
$$
The corollary now follows from the ``base $\times$ height'' formula for volume.
\end{proof}

We are now ready to present the proof of Theorem \ref{t:loss}. Recall that
our  goal is to find the expected loss  $\E[L(\of{\w^*}S)]$, where $S$ 
is a size $d$ volume sampled set.
\begin{proofof}{Theorem}{\ref{t:loss}}
First, we rewrite the expectation as follows:
\begin{align}
\E[L(\of{\w^*}S)] &= \sum_{S,|S|=d} P(S) L(\of{\w^*}S)
=\sum_{S,|S|=d} P(S) \sum_{j=1}^n \ell_j(\of{\w^*}S)\nonumber\\
&=\sum_{S,|S|=d}\sum_{j\notin S} P(S)\;\ell_j(\of{\w^*}S)
=\sum_{T,|T|=d+1}\sum_{j\in T}P(T_{-j})\;\ell_j(\of{\w^*}{T_{-j}}). 
\label{eq:sum-swap}
\end{align}
We now use Corollary \ref{cor:geometry} on the matrix $\X_T$ and 
test row $\x_j^\top$ (assuming $\rank(\X_{T_{-j}})=d$):
\begin{align}
\label{eq:summand}
P(T_{-j})\;\ell_j(\of{\w^*}{T_{-j}}) =
\frac{\det(\X_{T_{-j}}^\top\X_{T_{-j}})}{\det(\X^\top\X)}\;\ell_j(\of{\w^*}{T_{-j}}) =
\frac{\det(\Xs_T^\top\Xs_T)}{\det(\X^\top\X)}.
\end{align}
Since the summand does not depend on the index $j\in T$,
the inner summation in (\ref{eq:sum-swap}) becomes a multiplication
by $d+1$.  This lets us write the expected loss as:  
\begin{align}
\label{eq:th-cauchy-binet}
\!\!
\E[L(\of{\w^*}S)] = \frac{d+1}{\det(\X^\top\X)}
\!\sum_{T,|T|=d+1}\!\!\!\det(\Xs_T^\top\Xs_T)  
\overset{(1)}{=} (d+1)\frac{\det(\Xs^\top\Xs)}{\det(\X^\top\X)}
\overset{(2)}{=} (d+1)\,L(\w^*), 
\end{align}
where (1) follows from the Cauchy-Binet formula
and (2) is an application of the ``base $\times$ height'' formula.
If $\X$ is not in general position, then for some summands in \eqref{eq:summand},
$\rank(\X_{T_{-j}})<d$ and $P(T_{-j})=0$.
Thus the left-hand side of \eqref{eq:summand} is $0$, while the right-hand
side is non-negative, so \eqref{eq:th-cauchy-binet} becomes an inequality,
completing the proof of Theorem \ref{t:loss}.  
\end{proofof}

\subsection*{Lifting expectations to matrix form}
\label{sec:eprojs-proof}

We show the matrix expectation formula of Theorem \ref{t:eprojs} as a
corollary to the loss expectation formula of Theorem \ref{t:loss}. The
key observation is that the loss formula holds for
arbitrary response vector $\y$, which allows us to ``lift'' it to the
matrix form. 
\begin{proofof}{Theorem}{\ref{t:eprojs}}
Note, that the
loss of least squares estimator can be written in terms of the
projection matrix $\P_\X$:
\begin{align*}
L(\w^*) = \|\y - \ybh\|^2 = \|(\I - \P_\X)\y\|^2 =
  \y^\top(\I-\P_\X)^2\y \overset{(*)}{=}\y^\top(\I-\P_\X)\,\y, 
\end{align*}
where in $(*)$ we used the following property of a projection
matrix: $\P_\X^2=\P_\X$. Writing the loss expectation of the
subsampled estimator in the same form, we obtain:
\begin{align*}
\E[L(\of{\w^*}S)] &= \E[\|\y-\ybh(S)\|^2] = \E[\|(\I -
  \X(\I_S\X)^{+})\y\|^2] \\
&=\E[\y^\top(\I - \X(\I_S\X)^{+})^2\,\y] = \y^\top\E[(\I -
  \X(\I_S\X)^{+})^2]\,\y.
\end{align*}
Crucially, we are able to extract the response vector $\y$ out of the
expectation formula, which allows us to write the formula from Theorem
\ref{t:loss} as follows:
\begin{align*}
\y^\top\ \Blue{\E[(\I - \X(\I_S\X)^{+})^2]}\ \y = \y^\top
  \Blue{(d+1)(\I-\P_\X)}\ \y, \quad
\forall\, \y\in \R^n.
\end{align*}
We now use the following elementary fact: If for two
symmetric matrices $\A$ and $\B$, we have
$\y^\top \A\y=\y^\top\B\y,\;\forall\y\in\R^n$, then $\A=\B$.
This gives the matrix expectation formula:
\begin{align*}
\Blue{\E[(\I - \X(\I_S\X)^{+})^2]} = 
\Blue{(d+1)(\I-\P_\X)}.
\end{align*}
Expanding square on the l.h.s. of the above and applying Theorem \ref{t:einv}, 
we obtain the covariance-type equivalent form stated in 
Theorem \ref{t:eprojs}:
\begin{align*}
\I - 2\,\overbrace{\E[\X(\I_S\X)^{+}]}^{\P_\X} + \E[(\X(\I_S\X)^{+})^2]   
&=
(d+1)(\I-\P_\X)\\ 
\Longleftrightarrow\qquad \E[(\X(\I_S\X)^{+})^2] - \P_\X
&= d\,(\I-\P_\X). 
\end{align*}
\vspace{-1.3cm}

\end{proofof}



\subsection{Averaging unbiased estimators and the open problem
for worst-case responses}
\label{sec:av}


As discussed at the beginning of Section
\ref{sec:linear-regression}, our goal is to find a way to sample a
small index set $S$ and construct a weight 
function $\w(S)$ which uses responses $\y_S$ so that $\E[L(\w(S))]\leq
(1+c)\;L(\w^*)$, where the
multiplicative factor $1+c$ is bounded for all input
matrices $\X$ and all response vectors $\y$. 
Recall that $L(\cdot)$ denotes the square loss on all
rows and $\w^*$ is the optimal solution based on all responses.
We show in the previous subsections that the smallest size of $S$ for
which this goal can be achieved is $d$ (There is no
sampling procedure for sets of size less than $d$ and
weight function $\w(S)$ for which this factor is finite).
We also prove that when sets $S$ of size $d$ are drawn
proportional to the squared volume of $\X_S$ 
(i.e. $\det(\X_S^\top\X_S)$), then $\E[L(\of{\w^*}S)]\le (d+1) L(\w^*)$,
where the factor $d+1$ is optimal for some $\X$ and $\y$.
Here $\w^*(S)$ denotes the linear least squares solution for 
the subproblem $(\X_S,\y_S)$.


A natural more general goal is to get arbitrarily close to the
optimum loss. That is, for any $\epsilon$, what is the smallest sample size
$|S|=s$ for which there is a sampling distribution over subsets $S$ and
a weight function $\w(S)$ built from $\X$ and 
$\y_S$, such that $\E[L(\w(S))]\leq (1+\epsilon) \,L(\w^*)$.
A related bound for i.i.d. leverage score sampling states that a sample
size of $O(d\log d + \frac{d}{\epsilon})$ suffices to achieve a
$1+\epsilon$ factor \textit{with high probability} \citep{Hsuprivate,thethesis},
however this does not imply multiplicative bounds
\textit{in expectation}.\footnote{%
Also, the weight vectors produced from i.i.d. leverage score
sampling are not unbiased.}


We conjecture that some form of volume sampling can be used to achieve the
$1+\epsilon$ factor with sample size $O(\frac{d}{\epsilon})$, in expectation.
How close can we get with the techniques presented in this paper?
We showed that size $d$ volume sampling achieves a factor
of $1+d$, but we do not know how to generalize this proof
to sample size larger than $d$. However, one unique
property of the volume-sampled estimator $\w^*(S)$ that can be
useful here is that it is an {\em unbiased estimator} of $\w^*$. As
we shall see now, this basic property has many benefits.
For any unbiased estimator (i.e. $\E[\w(S)] =\w^*$) and
optimal prediction vector $\ybh=\X\w^*$,
consider the following rudimentary version of a
bias-variance decomposition:
\begin{equation}
\label{e:var}
\E\underbrace{\|\X\,\w(S)-\y\|^2}_{L(\w(S))} 
=\E\,\|\X\,\w(S)-\ybh+\ybh - \y\|^2
=\E\,\|\X\,\w(S) - \ybh\|^2 
   +\underbrace{\|\ybh-\y\|^2}_{L(\w^*)}.
\end{equation}
The unbiasedness of the estimator assures that the cross term
$(\X\overbrace{\E[\w(S)]}^{\w^*} 
- \ybh)^\top (\ybh-\y)$ is 0.
Therefore a $1+c$ factor loss bound is equivalent
to a $c$ factor variance bound, i.e.
\begin{align}
\overbrace{\E[L(\w(S))] \leq (1+c)\, L(\w^*)}^{\text{loss bound}}
\quad \Longleftrightarrow \quad
\overbrace{\E\,\|\X\,\w(S) - \ybh\|^2 \leq
  c\,L(\w^*)}^{\text{variance bound}}.
\label{eq:loss-variance}
\end{align}
To reduce the variance of any unbiased estimator $\w(S)$ 
(i.e. $\E[\w(S)] = \w^*$)
with sample size $s$, we can draw $k$ independent samples $S_1,\ldots,S_k$
of size $s$ each and predict with the average estimator
$\frac{1}{k}\sum_{j=1}^k \w(S_j)$. If the loss bound from
\eqref{eq:loss-variance} holds for $\w(S)$, then the average estimator satisfies
$$
\E \bigg[L\Big(\frac{1}{k}\sum\nolimits_j \w(S_j) \Big)\bigg]
\leq \Big(1 + \frac{c}{k}\Big)\,L(\w^*).
$$
Setting $k=c/\epsilon$, we need $t=s\,c/\epsilon$ responses to get a
$1+\epsilon$ approximation.
We showed that size $s=d$ volume sampling achieves factor $c=d$.
So with our current proof techniques, we need
$t=d^2/\epsilon$ responses to get a $1+\epsilon$ factor
approximation, for $\epsilon\in (0,d]$.%
\footnote{{Thus when averaging the
estimators of $k=t/d$ independent volume sampled sets of size $d$,
$$\underbrace{\E \bigg[L\Big(\frac{1}{k}\sum\nolimits_j \of{\w^*}{S_j} \Big)\bigg]
-L(\w^*)}_{\text{regret}}
= \underbrace{\frac{d^2 L(\w^*)}{t}}_{\text{prediction variance}},\quad \text{when $\X$ is
in general position}.$$
}}

The basic open problem for worst-case responses is the following:
Is there a size $O(d/\epsilon)$ {\em unbiased estimator} that
achieves a $1+\epsilon$ factor approximation?\footnote{%
In a recent paper \citep{chen2017condition} 
a $1+\epsilon$ factor approximation has been
achieved with $O(d/\epsilon)$ examples (for $\epsilon\in (0,1]$), 
but the guarantee holds with high
probability (and not in expectation) and the estimator is not unbiased.}
By the above averaging method this is equivalent to the following question:
Is there a size $O(d)$ unbiased estimator that
achieves a constant factor? This is because
once we have an unbiased estimator that achieves a
constant factor, then by averaging $1/\epsilon$ copies, we get the
$1+O(\epsilon)$ factor.
Ideally the special unbiased estimators resulting from a
version of volume sampling can achieve this feat. We
conclude this section with
our favorite open problem: Is there a version of 
$O(d)$ size volume sampling that achieves a constant factor approximation?

In the next section we make some
minimal statistical assumptions on the response vector
which let us prove much stronger bounds:
We assume that the response vector is linear plus bounded noise of mean zero.
In particular we show that with this noise model, 
$O(d)$ size volume sampling achieves a constant factor approximation.

\section{Regularized volume sampling for learning with
noisy responses}
\label{sec:regularized}

\begin{wrapfigure}{R}{0.4\textwidth}
\renewcommand{\thealgorithm}{}
\vspace{-8mm}
\hspace{-2mm}
\begin{minipage}{0.4\textwidth}
\floatname{algorithm}{}
\begin{algorithm}[H] 
  \caption{\bf \small $\lambda$-regularized size $s$ volume sampling}
  \begin{algorithmic}[0]
    \STATE $S \leftarrow [n]$
\vspace{1mm}
    \STATE {\bf while} $|S|>s$
    \STATE \quad $\forall_{i\in S}\!:\ h_i\leftarrow
    \frac{\det(\X_{\Sm}^\top\!\X_{\Sm}+\lambda\I)}{\det(\X_S^\top\X_S+\lambda\I)}$
\vspace{1mm}
    \STATE \quad Sample $i\propto h_i$ out of $S$
\vspace{1mm}
    \STATE \quad $S\leftarrow S - \{i\}$
    \STATE {\bf end} 
    \RETURN $S$
  \end{algorithmic}
\end{algorithm}
\end{minipage}
\vspace{-4mm}
\end{wrapfigure}

Volume sampling, as defined in Section \ref{sec:vol-def}, has certain
fundamental limitations. Namely, it is undefined whenever matrix $\X$
is not full rank or if we wish to sample a subset $S$ of size
smaller than the dimension $d$. Motivated by these limitations, we
propose a regularized variant, called $\lambda$-regularized
volume sampling, which we define through a generalization of the
reverse iterative sampling procedure:
\begin{align}
P(\Sm\,|\,S)  \propto
  \frac{\det(\X_\Sm^\top\X_\Sm+\lambda\I)}{\det(\X_S^\top\X_S+\lambda\I)}.
\label{eq:regularized}
\end{align}
The normalization factor of this conditional probability (i.e. the sum of
\eqref{eq:regularized} over $i\in S$) can be computed using
Sylvester's theorem:
\begin{align}
\sum_{i\in  S}
\frac{\det(\X_\Sm^\top\X_\Sm+\lambda\I)}{\det(\X_S^\top\X_S+\lambda\I)}
&= \sum_{i\in  S}\big(1 -
  \x_i^\top(\X_S^\top\X_S+\lambda\I)^{-1}\x_i\big)\nonumber\\
&= |S| - \tr\big(\X_S(\X_S^\top\X_S+\lambda\I)^{-1}\X_S^\top) \nonumber\\
&= |S| - d + \lambda\,\tr\big((\X_S^\top\X_S+\lambda\I)^{-1}\big).\label{eq:reg-sum}
\end{align}
Note that in the special case of no regularization (i.e.
$\lambda=0$) the last trace vanishes and \eqref{eq:reg-sum} is
equal to $|S|-d$, so we recover volume sampling from Section
\ref{sec:vol-def}. 
However, when $\lambda > 0$, then the last term is non-zero and
depends on the entire matrix $\X_S$. This makes regularized volume
sampling  more complicated and certain equalities proven in previous
sections for $\lambda=0$ no longer hold. In particular, the
analogous closed form of the sampling probability $P(S)$ given in Theorem
\ref{t:vol} is not recovered because the paths from node
$\{1..n\}$ to node $S$ in the graph of Figure \ref{fig:dag-sampling}
do not all have the same probability.
However, the proof technique we
developed for reverse iterative sampling can still be applied,
resulting in the following extension of the variance 
formula of Theorem \ref{t:einvs}:

\begin{theorem}\label{t:reg-einvs}
For any $\X\in\R^{n\times d}$, $\lambda\geq 0$, 
let $S$ be sampled according  to
$\lambda$-regularized size $s$ volume sampling from $\X$. Then,
\[\E\big[ (\X_S^\top\X_S+\lambda\I)^{-1}\big] \preceq
  \frac{n-d_\lambda+1}{s-d_\lambda+1}(\X^\top\X+\lambda\I)^{-1}\]
for any $s\geq d_\lambda\defeq \tr(\X(\X^\top\X+\lambda\I)^{-1}\X^\top)$.
\end{theorem}
\begin{remark}
Constant $d_\lambda$ is a common notion of statistical dimension often
referred to as the effective degrees of freedom. If $\lambda_i$ are
the eigenvalues of $\X^\top\X$, then $d_\lambda=\sum_{i=1}^d
\frac{\lambda_i}{\lambda_i+\lambda}$.
Note that $d_\lambda$ is decreasing with $\lambda$ and, when $\X$ is
full rank, $d_0=d$. Thus, unlike Theorem \ref{t:einvs}, the above
result  offers meaningful bounds for sampling sets $S$ of size smaller than $d$.
\end{remark}
\proof\ 
To obtain Theorem \ref{t:reg-einvs}, we use essentially the same methodology as
described in Lemma \ref{l:key}, except in the regularized case
equality is replaced with inequality. Recall that using Sylvester's
theorem  we can compute the unnormalized conditional probability from
\eqref{eq:regularized} as:
\begin{align*}
h_i&=\frac{\det(\X_\Sm^\top\X_\Sm+\lambda\I)}
        {\det(\X_S^\top\X_S+\lambda\I)}= 1-\x_i^\top(\X_S^\top\X_S+\lambda\I)^{-1}\x_i.
\end{align*}
From now on, we will use
$\of{\Z_\lambda}S=\X_S^\top\X_S+\lambda\I$ as a shorthand in the proofs.
Next, letting $M=\sum_{i\in S}h_i$, we compute
unnormalized expectation by applying the Sherman-Morrison formula:
\begin{align*}
M\, \E\big[(\X_\Sm^\top\X_\Sm+\lambda\I)^{-1}\,|\,S\big]
&=\sum_{i\in S} h_i \of{\Z_\lambda}\Sm^{-1}
= \sum_{i\in S} h_i\left(\of{\Z_\lambda}S^{-1} +
                                         \frac{\of{\Z_\lambda}S^{-1}\x_i\x_i^\top
                                         \of{\Z_\lambda}S^{-1}}{1-\x_i^\top\of{\Z_\lambda}S^{-1}\x_i}\right)\\
&=M\, \of{\Z_\lambda}S^{-1} + \of{\Z_\lambda}S^{-1}\Big(\sum_{i\in S}\x_i\x_i^\top\Big) \of{\Z_\lambda}S^{-1}\\
&=M\, \of{\Z_\lambda}S^{-1} + \of{\Z_\lambda}S^{-1}(\of{\Z_\lambda}S -
  \lambda\I)\of{\Z_\lambda}S^{-1}\\
&=M\, \of{\Z_\lambda}S^{-1} + \of{\Z_\lambda}S^{-1} -
  \lambda\of{\Z_\lambda}S^{-2} \\
&\preceq (M+1)\, \of{\Z_\lambda}S^{-1}.
\end{align*}
Finally, the normalization factor $M$ (which we already computed in
\eqref{eq:reg-sum}) can be lower-bounded using the $\lambda$-statistical dimension
$d_\lambda$ of matrix $\X$:
\begin{align*}
M &= \sum_{i\in S}(1-\x_i^\top\of{\Z_\lambda}S^{-1}\x_i) 
=s - d + \lambda\,\tr(\of{\Z_\lambda}S^{-1})
\geq s-\big(\underbrace{d \!-\! \lambda\,\tr(\of{\Z_\lambda}{\{1..n\}}^{-1})}_{d_\lambda}\big).
\end{align*}
Putting the bounds together, we obtain that:
\begin{align*}
\E\big[(\X_\Sm^\top\X_\Sm+\lambda\I)^{-1}\,|\,S\big]\preceq 
  \frac{s-d_\lambda+1}{s-d_\lambda} (\X_S^\top\X_S+\lambda\I)^{-1}.
\end{align*}
To prove Theorem \ref{t:reg-einvs} it remains to chain the
conditional expectations along the sequence of subsets obtained by
$\lambda$-regularized volume sampling:
\begin{align*}
\hspace{1.5cm}\E\big[\of{\Z_\lambda}S^{-1}\big] &\preceq
\left(\prod_{t=s+1}^n\frac{t-d_\lambda+1}{t-d_\lambda}\right)\,
\of{\Z_\lambda}{\{1..n\}}^{-1}
=\frac{n-d_\lambda+1}{s-d_\lambda+1}(\X^\top\X+\lambda\I)^{-1}.\hspace{.7cm}\BlackBox
\end{align*}

\subsection{Ridge regression with noisy responses}

We apply the above result to obtain statistical guarantees for
subsampling with regularized estimators.
Given a matrix $\X\in \R^{n\times d}$, we consider the
task of fitting a linear model
to a vector of responses
$\y=\X\wbt+\xib$, where $\wbt\in\R^d$ and the noise
$\xib\in\R^n$ is a
mean zero random vector with covariance matrix $\Var[\xib]\preceq\sigma^2\I$
for some $\sigma>0$. A classical solution to this task is the ridge estimator:
\begin{align*}
\w_\lambda^* &= 
\argmin_{\w\in\R^d} \; 
\|\X\w - \y\|^2 + \lambda\|\w\|^2=(\X^\top\X+\lambda\I)^{-1}\X^\top\y.
\end{align*}
As a consequence of Theorem \ref{t:reg-einvs}, we show that if $S$ is
sampled with $\lambda$-regularized volume sampling from $\X$, then the
ridge  estimator for the subproblem  $(\X_S,\y_S)$
$$
\of{\w_\lambda^*} S = (\X_S^\top\X_S+\lambda\I)^{-1}\X_S^\top\y_S 
$$
has strong generalization properties with respect to the full problem
$(\X,\y)$ in terms of the mean squared prediction error (MSPE) and
mean squared error (MSE). 

\begin{theorem}\label{thm:ridge}
Let $\X\in\R^{n\times d}$ and $\wbt\in\R^d$, and suppose that
$\y=\X\wbt + \xib$, where $\xib$ is a mean zero vector with
$\Var[\xib]\preceq\sigma^2\,\I$. Let $S$ be sampled according to
$\lambda$-regularized size $s\geq d_\lambda$ volume sampling from $\X$ and
$\of{\w_\lambda^*}S$ be 
the $\lambda$-ridge estimator of $\wbt$ computed from subproblem
$(\X_S,\y_S)$. Then, if $\lambda\leq \frac{\sigma^2}{\|\wbt\|^2}$, we have
\begin{align*}
\text{(mean squared prediction error)}&&
  \E_S\E_{\xib}\Big[\frac{1}{n}\|\X(\of{\w_\lambda^*}S-\wbt)\|^2\Big]
&\leq\frac{\sigma^2 d_\lambda}{s-d_\lambda+1},\\
\text{(mean squared error)}&&\
\E_S\E_{\xib}\big[\|\of{\w_\lambda^*}S-\wbt\|^2\big] &\leq
  \frac{\sigma^2n\,\tr((\X^\top\X+\lambda\I)^{-1})}{s-d_\lambda+1}.
\end{align*}
\end{theorem}

Next, we present two lower-bounds for MSPE of a
subsampled ridge estimator which show that the statistical guarantees
achieved by regularized volume sampling are
nearly optimal for $s\gg d_\lambda$ and better than standard
approaches for $s=O(d_\lambda)$. In particular, we show that 
non-i.i.d.~nature of volume sampling is essential
if we want to achieve good generalization when the number of
responses is close to $d_\lambda$. 
Namely, for certain data matrices any i.i.d. subsampling procedure
(such as i.i.d. leverage score sampling) 
requires more than $d_\lambda\ln(d_\lambda)$ responses to
achieve MSPE below $\sigma^2$. In contrast volume
sampling obtains that bound for any matrix with $2d_\lambda$ responses.

\begin{theorem}\label{thm:lb-all}
For any $p\geq 1$ and $\sigma\geq 0$, there is $d\geq p$ such
that for any sufficiently large $n$ divisible by $d$ there exists a
matrix $\X\in\R^{n\times d}$ such that 
\[d_\lambda(\X)\geq p\quad \text{ for any }\quad 0\leq\lambda\leq \sigma^2,\] 
and for each of the following two statements there is a vector
$\wbt\in\R^d$ for which the corresponding
 regression problem $\y=\X\wbt+\xib$ with $\Var[\xib]=\sigma^2\I$
satisfies that statement:
\begin{enumerate}
\item For any subset $S\subseteq\{1..n\}$ of size $s$,
\begin{align*}
  \E_{\xib}\Big[\frac{1}{n}\|\X(\of{\w_\lambda^*}S-\wbt)\|^2 \Big]
&\geq\frac{\sigma^2 d_\lambda}{s+d_\lambda};
\end{align*}
\item For multiset $S\subseteq\{1..n\}$ of size $s\leq
  (d_\lambda\!-\!1)\ln(d_\lambda)$, sampled i.i.d. from any distribution
  over $\{1..n\}$,
\begin{align*}
\E_S\E_{\xib}\Big[\frac{1}{n}\|\X(\of{\w_\lambda^*}S-\wbt)\|^2\Big] \geq \sigma^2.
\end{align*}
\end{enumerate}
\end{theorem}

\begin{proofof}{Theorem}{\ref{thm:ridge}}
Standard analysis for the ridge regression estimator follows by
performing bias-variance decomposition of the error, and then
selecting $\lambda$ so that bias can be appropriately bounded. We
will recall this calculation for a fixed subproblem
$(\X_S,\y_S)$. First, we
compute the bias of the ridge estimator for a fixed set $S$ (recall
the shorthand $\of{\Z_\lambda}S=\X_S^\top\X_S+\lambda\I$):
\begin{align*}
\text{Bias}_{\xib}[\of{\w_\lambda^*}S] &= \E[\of{\w_\lambda^*}S] - \wbt
=\E_{\xib}\,[\of{\Z_\lambda}S^{-1}\X_S^\top\y_S] - \wbt \\
&=\of{\Z_\lambda}S^{-1}\X_S^\top\,(\X_S\wbt + \cancel{\E_{\xib}[\xib_S]}) - \wbt\\
&=(\of{\Z_\lambda}S^{-1}\X_S^\top\X_S - \I)\wbt
= -\lambda \,\of{\Z_\lambda}S^{-1}\wbt.
\end{align*}
Similarly, the covariance matrix of $\of{\w_\lambda^*}S$ is given by:
\begin{align*}
\Var_{\xib}[\of{\w_\lambda^*}S] &=
  \of{\Z_\lambda}S^{-1}\X_S^\top\Var_{\xib}[\xib_S]\X_S
  \of{\Z_\lambda}S^{-1}\\
&\preceq \sigma^2\of{\Z_\lambda}S^{-1}\X_S^\top\X_S
  \of{\Z_\lambda}S^{-1}
=\sigma^2(\of{\Z_\lambda}S^{-1}-\lambda\,
  \of{\Z_\lambda}S^{-2}).
\end{align*}
Mean squared error of the ridge estimator for a fixed subset $S$ can
now be bounded by:
\begin{align}
\E_{\xib}\big[\|\of{\w_\lambda^*}S - \wbt\|^2\big] &=
  \tr(\Var_{\xib}[\of{\w_\lambda^*}S]) +                                                         \|\text{Bias}_{\xib}[\of{\w_\lambda^*}S]\|^2\nonumber\\ 
&\leq\sigma^2\tr(\of{\Z_\lambda}S^{-1}\!\!\!-\lambda
  \of{\Z_\lambda}S^{-2}) +
  \lambda^2\tr(\of{\Z_\lambda}S^{-2}\wbt\wbt^{\top})\nonumber\\
&\leq \sigma^2\tr(\of{\Z_\lambda}S^{-1}) +
  \lambda\tr(\of{\Z_\lambda}S^{-2})(\lambda\|\wbt\|^2\!- \sigma^2)\label{eq:cauchy-trace}\\
&\leq \sigma^2\tr(\of{\Z_\lambda}S^{-1}),\label{eq:lambda-bound}
\end{align}
where in (\ref{eq:cauchy-trace}) we applied Cauchy-Schwartz inequality for matrix trace, and
in (\ref{eq:lambda-bound}) we used the assumption that $\lambda\leq
\frac{\sigma^2}{\|\wbt\|^2}$. Thus, taking expectation over the sampling of set $S$, we get
\begin{align}
\E_S\E_{\xib}\big[\|\of{\w_\lambda^*}S - \wbt\|^2\big] &\leq
  \sigma^2\E_S\big[\tr(\of{\Z_\lambda}S^{-1})\big] \nonumber \\
\text{(Theorem \ref{t:reg-einvs})}\quad&\leq\sigma^2
  \frac{n-d_\lambda+1}{s-d_\lambda+1}
  \tr(\of{\Z_\lambda}{\{1..n\}}^{-1}) \label{eq:ridge-mse}\\
&\leq \frac{\sigma^2\,n\, \tr((\X^\top\X+\lambda\I)^{-1})}{s-d_\lambda+1}. \nonumber
\end{align}
Next, we bound the mean squared prediction error. As before, we start
with the standard  bias-variance decomposition for fixed set $S$:
\begin{align*}
\E_{\xib}\big[\|\X(\of{\w_\lambda^*}S \!-\! \wbt)\|^2\big] \!
&=
  \tr(\Var_{\xib}[\X\of{\w_\lambda^*}S]) + \|\X(\E_{\xib}[\of{\w_\lambda^*}S]
  - \wbt)\|^2\\
&\leq\sigma^2\tr(\X(\of{\Z_\lambda}S^{\!-1}\!\!-\!\lambda\,
  \of{\Z_\lambda}S^{\!-2})\X^\top)
+
  \lambda^2\tr(\of{\Z_\lambda}S^{\!-1}\X^\top\X \of{\Z_\lambda}S^{\!-1}\wbt\wbt^{\top})\\
&\leq \sigma^2\tr(\X\of{\Z_\lambda}S^{-1}\X^\top)
+ \lambda\,\tr(\X\of{\Z_\lambda}S^{-2}\X^\top)(\lambda\|\wbt\|^2
  - \sigma^2)\\
&\leq \sigma^2\tr(\X\of{\Z_\lambda}S^{-1}\X^\top).
\end{align*}
Once again, taking expectation over subset $S$, we have
\begin{align}
\E_S\E_{\xib}\Big[\frac{1}{n}\|\X(\of{\w_\lambda^*}S -
  \wbt)\|^2\Big]&
\leq\frac{\sigma^2}{n}
  \E_S\big[\tr(\X\of{\Z_\lambda}S^{-1}\X^\top)\big]
=\frac{\sigma^2}{n}\tr(\X \,\E_S[\of{\Z_\lambda}S^{-1}]\,\X^\top)\nonumber\\
\text{(Theorem \ref{t:reg-einvs})}\quad &\leq \frac{\sigma^2}{n} \frac{n-d_\lambda+1}{s-d_\lambda+1}\tr(\X
  \of{\Z_\lambda}{\{1..n\}}^{-1}\X^\top)
 \leq \frac{\sigma^2 d_\lambda}{s-d_\lambda+1}. \label{eq:ridge-mspe}
\end{align}
The key part of proving both bounds is the application of Theorem
\ref{t:reg-einvs}. For MSE, we only used the trace version of the
inequality (see (\ref{eq:ridge-mse})), however to obtain the bound on
MSPE we used the more general positive semi-definite inequality in
(\ref{eq:ridge-mspe}).
\end{proofof}

\begin{proofof}{Theorem}{\ref{thm:lb-all}}
Let $d=\lceil p\rceil+1$ and $n\geq\lceil \sigma^2\rceil d(d-1)$ be
divisible by $d$. We define 
\begin{align*}
\X &\defeq [\I,...,\I]^\top\in\R^{n\times d},\qquad
\wbt^{\top}\defeq \,[a\sigma,...,a\sigma]\in\R^d
\end{align*}
 for some $a>0$. For any  $\lambda\leq \sigma^2$, the
 $\lambda$-statistical dimension of $\X$ is 
\begin{align*}
d_\lambda &= \tr(\X\,\of{\Z_\lambda}{\{1..n\}}^{-1}\X^\top)
\geq\frac{\lceil \sigma^2\rceil d(d-1)}{\lceil \sigma^2\rceil (d-1) +
            \lambda}\geq \frac{d(d-1)}{d-1+1}\geq p.
\end{align*}
Let $S\subseteq\{1..n\}$ be any set of size $s$, and for $i\in\{1..d\}$ let 
$s_i \defeq |\{i\in S:\, \x_i=\e_i\}|$.
The prediction variance of estimator $\of{\w_\lambda^*}S$ is equal to
\begin{align*}
\tr\big(\Var_{\xib}[\X\of{\w_\lambda^*}S]\big)
&=\sigma^2\tr(\X(\of{\Z_\lambda}S^{-1}\!-\lambda \of{\Z_\lambda}S^{-2})\X^\top)\\
&=\frac{\sigma^2n}{d}\sum_{i=1}^d\left(\frac{1}{s_i+\lambda} -
  \frac{\lambda}{(s_i+\lambda)^2}\right)
=\frac{\sigma^2n}{d}\sum_{i=1}^d\frac{s_i}{(s_i+\lambda)^2}.
\end{align*}
The prediction bias of estimator $\of{\w_\lambda^*}S$ is equal to
\begin{align*}
\|\X(\E_{\xib}[\of{\w_\lambda^*}S]-\wbt)\|^2
&=\lambda^2\wbt^{\top}\of{\Z_\lambda}S^{-1}
\X^\top\X\of{\Z_\lambda}S^{-1}\wbt \\
&=\frac{\lambda^2a^2\sigma^2n}{d}\,\tr\big(\of{\Z_\lambda}S^{-2}\big)
=\frac{\lambda^2a^2\sigma^2n}{d}\sum_{i=1}^d\frac{1}{(s_i+\lambda)^2}.
\end{align*}
Thus, MSPE of estimator  $\of{\w_\lambda^*}S$
is given by:
\begin{align*}
\E_{\xib}\Big[\frac{1}{n}\|\X(\of{\w_\lambda^*}S-\wbt)\|^2\Big]
&=\frac{1}{n}\tr\big(\Var_{\xib}[\X\of{\w_\lambda^*}S]\big)
  +\frac{1}{n}\|\X(\E_{\xib}[\of{\w_\lambda^*}S]-\wbt)\|^2\\
&=\frac{\sigma^2}{d}\sum_{i=1}^d\left(\frac{s_i}{(s_i+\lambda)^2} +
  \frac{a^2\lambda^2}{(s_i+\lambda)^2}\right)
=\frac{\sigma^2}{d}\sum_{i=1}^d\frac{s_i +a^2\lambda^2}{(s_i+\lambda)^2}.
\end{align*}
Next, we find the $\lambda$ that minimizes this expression. Taking the
derivative with respect to $\lambda$ we get:
\begin{align*}
\frac{\partial}{\partial
  \lambda}\left(\frac{\sigma^2}{d}\sum_{i=1}^d\frac{s_i
  +a^2\lambda^2}{(s_i+\lambda)^2}\right)=\frac{\sigma^2}{d}\sum_{i=1}^d\frac{2s_i(\lambda
  - a^{-2})}{(s_i+\lambda)^3}.
\end{align*}

Thus, since at least one $s_i$ has to be greater than $0$, for any set
$S$ the derivative is negative  for $\lambda< a^{-2}$ and
positive for $\lambda>a^{-2}$, and the unique minimum of
MSPE is achieved at $\lambda=a^{-2}$, regardless of 
which subset $S$ is chosen. So, as we are seeking a lower bound, we can
focus on the case of $\lambda=a^{-2}$. 

\emph{Proof of Part 1.}
Let $a=1$. As shown above, we can assume that $\lambda=1$. In this
case the formula simplifies to:
\begin{align*}
\E_{\xib}\Big[\frac{1}{n}\|\X(\of{\w_\lambda^*}S-\wbt)\|^2\Big] &=\frac{\sigma^2}{d}\sum_{i=1}^d\frac{s_i+1}
    {(s_i+1)^2}
=\frac{\sigma^2}{d}\sum_{i=1}^d\frac{1}{s_i+1}\\
&\overset{(*)}{\geq}
  \frac{\sigma^2}{\frac{s}{d} +
  1}=\frac{\sigma^2d}{s+d}\geq \frac{\sigma^2d_\lambda}{s+d_\lambda},
\end{align*}
where $(*)$ follows by applying Jensen's inequality to convex function
$\phi(x)=\frac{1}{x+1}$. 

\emph{Proof of Part 2.}
Let $a=\sqrt{2d}$. As shown above, we can assume that
$\lambda=1/(2d)$. Suppose that multiset $S$ is sampled i.i.d. from some
distribution over set $\{1..n\}$. Similarly as in Corollary
\ref{cor:lb-iid},
we exploit the Coupon Collector's problem, i.e. that if $|S|\leq 
(d-1)\ln(d)$, then with probability at least
$1/2$ there is $i\in\{1..d\}$ such that $s_i=0$ (i.e., one of
the unit vectors $\e_i$ was never selected). Thus,
MSPE can be lower-bounded as follows:
\begin{align*}
\E_S\E_{\xib}\Big[\frac{1}{n}\|\X(\of{\w_\lambda^*}S-\wbt)\|^2\Big]
&\geq
  \frac{1}{2}\,\frac{\sigma^2}{d}\,
  \frac{s_i+a^2\lambda^2}{(s_i+\lambda)^2}=\frac{\sigma^2}{2d}
  \frac{2d \lambda^2}{\lambda^2} = \sigma^2.
\end{align*}
\vspace{-1.4cm}

\end{proofof}

\section{Efficient algorithms for volume sampling}
\label{sec:algorithm}

In this section we propose algorithms for efficiently performing
volume sampling. This addresses the
question posed by \cite{avron-boutsidis13}, asking for a
polynomial-time algorithm for the case when
the size of set $S$ is $s>d$. \cite{efficient-volume-sampling}
gave an algorithm for the case when $s=d$, which was later improved by \cite{more-efficient-volume-sampling}, running in time
$O(nd^3)$. Recently, \cite{dual-volume-sampling} offered an algorithm
for arbitrary $s$, which has complexity $O(n^4 s)$. We propose two new methods, which use
our reverse iterative sampling technique to achieve faster running times
for volume sampling of any size $s$. Both algorithms apply to the more general setting
of $\lambda$-regularized volume sampling (described in Section
\ref{sec:regularized}), and produce standard volume sampling as a
special case for $\lambda=0$ and $s\geq d$. The first algorithm has a
deterministic runtime of $O((n\!-\!s\!+\!d)nd)$, whereas the second one
is an accelerated version which with high probability finishes in time $O(nd^2)$.
Thus, we obtain a direct
improvement over \cite{dual-volume-sampling} by a factor of at least
$n^2$, and in the special case of $s=d$, by a factor of $d$ over the
algorithm of \cite{more-efficient-volume-sampling}.

Our algorithms implement reverse iterative sampling from
Theorem \ref{t:vol}. We start with the full index set
$S=\{1..n\}$. In one step of the algorithm, we remove one
row from set $S$. After removing $q$ rows, we are
left with the index set of size $n-q$ 
that is distributed according to volume sampling 
for row set size $n-q$, and we proceed until our set $S$
has the desired size $s$. The primary cost of the procedure is
updating the conditional distribution $P(\Sm|S)$ at every step. 
It is convenient to store it using
the unnormalized weights defined in \eqref{eq:regularized} which, via Sylvester's
theorem, can be computed as $h_i=1- \x_i^\top(\X_S^\top\X_S +
\lambda\I)^{-1}\x_i$ (For the sake of
generality  we state the methods for $\lambda$-regularized volume
sampling).
Doing this naively, we would first compute $(\X_S^\top\X_S + \lambda\I)^{-1}$
which takes $O(nd^2)$ time\footnote{We are primarily interested in the
  case where $n\geq d$ and we state our time bounds under that
  assumption. However, when $\lambda>0$, our techniques can be easily adapted to the
  case of $n<d$.}. After that for each $i$, we
would multiply this matrix by $\x_i$ in time $O(d^2)$ to get the $h_i$'s. 
The overall runtime of this naive method becomes:
\begin{align*}
\overbrace{\text{\# of steps}}^{n-s}\ \times
  \ (\,\overbrace{\text{compute }(\X_S^\top\X_S + \lambda\I)^{-1}}^{O(nd^2)}\ +\ \overbrace{\text{\# of weights}}^{\leq n}\ \times \ 
  \overbrace{\text{compute $h_i$}}^{O(d^2)}\,)\ =\ O((n-s)nd^2).
\end{align*}
Both the computation of matrix inverse and the weights $h_i$ can be 
made more efficient. First, the matrix
$(\X_S^\top\X_S+\lambda\I)^{-1}$ can be
computed from the one obtained in the previous step by using the
Sherman-Morrison formula. This lets us update it in $O(d^2)$ time
instead of $O(nd^2)$. Furthermore, we propose two strategies for dealing
with the cost of maintaining the weights:
\begin{enumerate}
\item Update all $h_i$'s at every step using Sherman-Morrison;
\item Use rejection sampling and only compute the $h_i$'s
needed for the rejection trials
(This avoids computing all $h_i$'s, but makes the
computation of each needed $h_i$ more expensive).
\end{enumerate}
As we can see, there is a trade-off between those
strategies. In the following lemma, we will show that updating the value of
$h_i$, given its value in the  previous step only costs $O(d)$ time as
opposed to $O(d^2)$. However, the number of $h_i$'s that need to be
computed for rejection sampling (explained shortly) can be far smaller.
\begin{lemma}\label{l:invariant}
For any matrix $\X\in\R^{n\times d}$, set $S\subseteq\{1..n\}$ and two
distinct indices $i,j\in S$, we have
\begin{align*}
1 - \x_j^\top(\X_\Sm^\top\X_\Sm\!+\lambda\I)^{-1}\x_j = 
 h_j - (\x_j^\top\v)^2,
\end{align*}
where $h_j=1- \x_j^\top(\X_S^\top\X_S +
\lambda\I)^{-1}\x_j$ and
$\v = \frac{1}{\sqrt{h_i}}(\X_S^\top\X_S+\lambda\I)^{-1}\x_i$.
\end{lemma}
\begin{proof}
Letting $\of{\Z_\lambda}S=\X_S^\top\X_S+\lambda\I$, we have
\begin{align*}
h_j - (\x_j^\top\v)^2 
&= 1- \x_j^\top \of{\Z_\lambda}S^{-1}\x_j -
      \frac{(\x_j^\top \of{\Z_\lambda}S^{-1}\x_i)^2}{1-\x_i^\top \of{\Z_\lambda}S^{-1}\x_i}\\
&=1- \x_j^\top \of{\Z_\lambda}S^{-1} \x_j -
  \frac{\x_j^\top \of{\Z_\lambda}S^{-1} \x_i\x_i^\top \of{\Z_\lambda}S^{-1}\x_j}
{1-\x_i^\top \of{\Z_\lambda}S^{-1}\x_i}\\
&=1 - \x_j^\top\left( \of{\Z_\lambda}S^{-1} +
  \frac{\of{\Z_\lambda}S^{-1}\x_i\x_i^\top \of{\Z_\lambda}S^{-1}}
  {1-\x_i^\top \of{\Z_\lambda}S^{-1}\x_i}\right)\x_j  \\
&\overset{(*)}{=}
1- \x_j^\top(\X_\Sm^\top\X_\Sm + \lambda\I)^{-1}\x_j,
\end{align*}
where $(*)$ follows from the Sherman-Morrison formula.
\end{proof}
Thus the overall time complexity of reverse iterative sampling when using the first
strategy goes down by a factor of $d$ compared to the naive version
(except for an initialization cost which stays at $O(nd^2)$).
\begin{theorem}\label{t:volsamp}
Algorithm \volsamp\ produces an index set $S$ of rows distributed
according to $\lambda$-regularized size
$s$ volume sampling over $\X$ in time $O((n\!-\!s\!+\!d)nd)$.
\end{theorem}
\begin{proof}
Using Lemma \ref{l:invariant} for $h_i$ and the Sherman-Morrison
formula for $\Z$, the following invariants hold at the beginning of
the {\bf while} loop: 
\begin{align*}
h_i = 1 - \x_i^\top(\X_S^\top\X_S+\lambda\I)^{-1}\x_i
\qquad
\text{and}
\qquad
\Z= (\X_S^\top\X_S + \lambda\I)^{-1}.
\end{align*}
%
%
Runtime: Computing the
initial $\Z=(\X^\top\X+\lambda\I)^{-1}$ takes $O(nd^2)$, as does
computing the initial values of $h_j$'s. Inside the \textbf{while}
loop, updating $h_j$'s takes $O(|S| d)=O(nd)$ and updating $\Z$ takes
$O(d^2)$. The overall runtime becomes $O(nd^2 + (n-s)nd) =
O((n-s+d)nd)$. 
\end{proof}

\vspace{5mm}
\noindent\begin{minipage}{\textwidth}
   \centering
\begin{minipage}{.45\textwidth}
\small
     \centering
\captionof{algorithm}{{\volsamp}($\X,s,\lambda$)}
  \begin{algorithmic}[1]
    \STATE $\Z\leftarrow (\X^\top\X+\lambda\I)^{-1}$\label{line:inv}   
    \STATE $\forall_{i\in\{1..n\}} \quad h_i\leftarrow 1-\x_i^\top \Z\x_i$
    \STATE $S \leftarrow \{1..n\}$
    \STATE {\bf while} $|S|>s$
    \STATE \quad Sample $i \propto h_i$ out of $S$
    \STATE \quad $S\leftarrow S - \{i\}$
    \STATE \quad $\v \leftarrow \Z\x_i /\sqrt{h_i}$
    \STATE \quad $\forall_{j\in S}\quad  h_j\leftarrow h_j -  (\x_j^\top\v)^2$
    \STATE \quad $\Z \leftarrow \Z + \v\v^\top$  
    \STATE {\bf end} 
    \RETURN $S$
 \end{algorithmic}
\label{alg:volsamp}
\vspace{7mm}
   \end{minipage}
   \begin{minipage}{.45\textwidth}
\small
     \centering 
     \captionof{algorithm}{{\lazy}($\X,s,\lambda$)}
\label{alg:lazyvolsamp}
\begin{algorithmic}[1]
  \STATE $\Z\leftarrow (\X^\top\X+\lambda\I)^{-1}$   
  \STATE $S \leftarrow \{1..n\}$
  \STATE {\bf while} $|S|>\max\{s,2d\}$
  \STATE \quad \textbf{repeat}\label{line:repeat}
  \STATE \quad\quad Sample $i$ uniformly out of $S$
  \STATE \quad\quad $h_i \leftarrow 1-\x_i^\top \Z\x_i$
\STATE \quad\quad Sample $A \sim \text{Bernoulli}(h_i)$
  \STATE \quad \textbf{until} $A=1$
  \STATE \quad $S\leftarrow S - \{i\}$
  \STATE \quad $\Z \leftarrow \Z + h_i^{-1}\Z\x_i\x_i^\top\Z$  
  \STATE {\bf end} 
  \STATE {\bf if} $s<2d$,\quad $S\leftarrow
  $ \volsamp($\X_S,s,\lambda$)\quad{\bf end}
  \RETURN $S$
\end{algorithmic}
   \end{minipage}
\end{minipage}
\vspace{4mm}

Next we present algorithm \lazy, which is based on the rejection
sampling strategy. Our key observation is that updating the full
conditional distribution $P(\Sm|S)$ is wasteful, since the distribution
changes very slowly throughout the procedure. Moreover, the
unnormalized weights $h_i$, which are computed in the process are all
bounded by 1. Thus, to sample from the correct distribution at any
given iteration, we can employ rejection sampling as follows: 
\begin{enumerate}
\setlength\itemsep{0mm}
\item Sample $i$ uniformly from set $S$,
\item Compute $h_i$,
\item Accept with probability $h_i$,
\item Otherwise, draw another sample.
\end{enumerate}
Note that this rejection sampling can be employed locally, within each
iteration of the algorithm. Thus, one rejection does not revert us
back to the beginning of the algorithm.
Moreover, if the probability of acceptance is high, then this strategy
requires computing only a small number of weights per iteration of the algorithm, as
opposed to updating all of them. This turns out to be the case for
a majority of the steps of the algorithm, except at the very end (for $s\leq 2d$), were the
conditional probabilities start changing more drastically. At that
point, it becomes more efficient to use the first algorithm, \volsamp.

\begin{theorem}
\label{thm:exact}
For any $\lambda,s\geq 0$, and $\delta\in(0,1)$, algorithm \lazy\ samples according to
$\lambda$-regularized size $s$ volume sampling, and with probability at least
$1-\delta$ runs in time
$$O\bigg(\Big(n+\log\big(n/d\big)
    \log\big(1/\delta\big)\Big)d^2\bigg).$$ 
\end{theorem}
\begin{proof}
We analyze the efficiency of rejection sampling in {\lazy}. Let $R_t$
be a random variable corresponding to the number of trials needed in
the \textbf{repeat} loop from line \ref{line:repeat} in {\lazy} at the point when $|S|=t$. 
Note that conditioning on the algorithm's history,
$R_t$ is distributed according to geometric
distribution $\text{Ge}(q_t)$ with success probability:
\begin{align*}
q_t = \frac{1}{t}\sum_{i\in S}\big(1-\x_i^\top (\X_S^\top\X_S+\lambda\I)^{-1}\x_i\big)\geq \frac{t-d}{t}\geq \frac{1}{2}.
\end{align*}
Thus, even though variables $R_t$ are not themselves independent, they
can be upper-bounded by a sequence of independent variables
$\Rh_t\sim \text{Ge}(\frac{t-d}{t})$. The expectation of the total
number of trials in {\lazy}, $\bar{R}=\sum_t R_t$, can thus be bounded
as follows: 
\begin{align*}
\E[\bar{R}]&\leq \sum_{t=2d}^n\E[\Rh_t]
             =\sum_{t=2d}^n\frac{t}{t-d}\leq 2n.
\end{align*}

Next, we will obtain a similar bound with
high probability instead of in expectation. Here, we will have to use
the fact that the variables $\Rh_t$ are independent, which means that
we can upper-bound their sum with high probability using standard
concentration bounds for geometric distribution. For example, using
Corollary 2.2 from 
\cite{geometric-tail-bounds} one can immediately show that with
probability at least $1-\delta$ we have $\bar{R} =
O(n\ln\delta^{-1})$. However, more careful analysis shows an even
better dependence on $\delta$. 

\begin{lemma}
\label{lm:tail-bound}
Let $\Rh_t\sim \textnormal{Ge}(\frac{t-d}{t})$ be independent random
variables. Then w.p. at least  $1-\delta$
\[\sum_{t=2d}^n \Rh_t = O\Big(n +
    \log\big(n/d\big)\log\big(1/\delta\big)\Big). \] 
\end{lemma}
Each trial of rejection sampling requires computing one weight $h_i$ in
time $O(d^2)$. The overall time complexity of \lazy\ thus includes computation
and updating of matrix $\Z$ (in time $O(nd^2)$), rejection sampling which
takes $O\left(\left(n+\log\left(\frac{n}{d}\right)\log\left(\frac{1}{\delta}\right)\right) d^2\right)$
time, and (if $s<2d$) the {\volsamp} portion, taking $O(d^3)$.
\end{proof}

\begin{proofof}{Lemma}{\ref{lm:tail-bound}}
As observed by \cite{geometric-tail-bounds}, tail-bounds for the sum
of geometric random variables depend on the minimum acceptance
probability among those variables. Note that for the vast majority of
$\Rh_t$'s the acceptance probability is very close to 1, so
intuitively we should be able to take advantage of this to improve our
tail bounds. To that end, we partition the variables into groups of
roughly similar acceptance probability and then separately bound the
sum of variables in each group. Let $J=\log(\frac{n}{d})$ (w.l.o.g. assume
that $J$ is an integer). For $1\leq j\leq J$, let
$I_j=\{d2^j, d2^j+1,..,d2^{j+1}\}$ represent the $j$-th
partition. We use the following notation for  each partition:
\begin{align*}
\bar{R}_j &\defeq \sum_{t\in I_j}R_t,
&  \mu_j &\defeq\E[\bar{R}_j],
&r_j &\defeq \min_{t\in I_j}\frac{t-d}{t}, 
& \gamma_j &\defeq \frac{\log(\delta^{-1})}{d2^{j-2}} + 3.
\end{align*}
Now, we apply Theorem 2.3 of \cite{geometric-tail-bounds} to
$\bar{R}_j$, obtaining
\begin{align*}
P(\bar{R}_j&\geq \gamma_j\mu_j) \leq
  \gamma_j^{-1}(1-r_j)^{(\gamma_j-1-\ln\gamma_j)\mu_j} 
\overset{(1)}{\leq}  (1-r_j)^{\gamma_j\mu_j/4}  \overset{(2)}{\leq} 2^{-j\gamma_jd2^{j-2}},
\end{align*}
where $(1)$ follows since $\gamma_j\geq 3$, and $(2)$ holds because
$\mu_j\geq d2^j$ and $r_j\geq 1-2^{-j}$. Moreover, for the chosen
$\gamma_j$ we have
\begin{align*}
j\gamma_jd2^{j-2} &= j\log(\delta^{-1}) + 3jd2^{j-2} 
\geq \log(\delta^{-1}) + j = \log(2^j\delta^{-1}).
\end{align*}
Let $A$ denote the event that $\bar{R}_j\leq \gamma_j\mu_j$ for all
$j\leq J$. Applying union bound, we get
\begin{align*}
P(A) &\geq 1- \sum_{j=1}^J
  P(\bar{R}_j\geq \gamma_j\mu_j) 
\geq 1- \sum_{j=1}^J
       2^{-\log(2^j\delta^{-1})}=1-\sum_{j=1}^J \frac{\delta}{2^j}\geq
  1-\delta.
\end{align*}
If $A$ holds, then we obtain the desired bound:
\begin{align*}
\sum_{t=2d}^n \Rh_t&\leq \sum_{j=1}^J \gamma_j\mu_j \leq
  \sum_{j=1}^J\left(\frac{\log(\delta^{-1})}{d2^{j-2}} +
  3\right)\,d2^{j+1} 
=8 J \log(\delta^{-1}) + 6\sum_{j=1}^Jd2^j\\
&= O\Big(\log\big(n/d\big)\log\big(1/\delta\big) +  n\Big).
\end{align*}
\vspace{-1.3cm}

\end{proofof}

\section{Experiments}
\label{sec:experiments}

In this section we experimentally evaluate the proposed
volume sampling algorithms in terms of runtime and
in the task of subsampling for linear regression. We use
regularization both for sampling and for prediction, as discussed in
Section \ref{sec:regularized}. The list of implemented algorithms is:
\begin{enumerate}
\item Regularized volume sampling (algorithms \lazy\ and \volsamp),
\item Leverage score sampling\footnote{Regularized variants of
    leverage scores have also been considered in context of kernel
    ridge regression \cite{ridge-leverage-scores}. However, in our experiments
    regularizing leverage scores did not provide any improvements.}
  (LSS) -- a popular i.i.d. sampling  technique
  \citep{randomized-matrix-algorithms}, where examples are selected
  w.p. $P(i) = (\x_i^\top(\X^\top\X)^{-1}\x_i)/d.$
\end{enumerate}
\begin{wrapfigure}{R}{0.55\textwidth}
\vspace{-.4cm}
\begin{tabular}{c|c|c|c|c}
Dataset & $n\times d$ & \!{\volsamp}\! &\!{\lazy}\! & LSS \\
\hline
cadata & $21$k$\times8$  & 33.5s & 0.9s&0.1s\\
MSD & \! 464k$\times$90\! & $>$24hr & 39s& 12s\\
\!cpusmall\! & 8k$\times$12&1.7s&0.4s&0.07s\\
abalone & 4k$\times$8 &0.5s&0.2s&0.03s
\end{tabular}
\captionof{table}{List of regression datasets with runtime comparison
  between {\volsamp} and {\lazy}. We also provide the 
  runtime for i.i.d. sampling with exact leverage scores (LSS).}  
\label{tab:datasets}
\end{wrapfigure}

The experiments were performed on several benchmark linear regression 
datasets from the libsvm repository \citep{libsvm}. Table
\ref{tab:datasets} lists those datasets along with running times
for sampling dimension many columns with each method.
Dataset MSD was too big for {\volsamp} to finish in reasonable time,
however {\lazy} finished in less than 40 seconds.
In Figure \ref{fig:time-plots} we plot the runtime against varying
values of $n$ (using portions of the datasets), to
compare how {\lazy} and {\volsamp} scale with
respect to the data size. We observe that {\lazy}
exhibits linear dependence on $n$, thus it is much better suited for
running on large datasets. 

\begin{figure}
\includegraphics[width=0.5\textwidth]{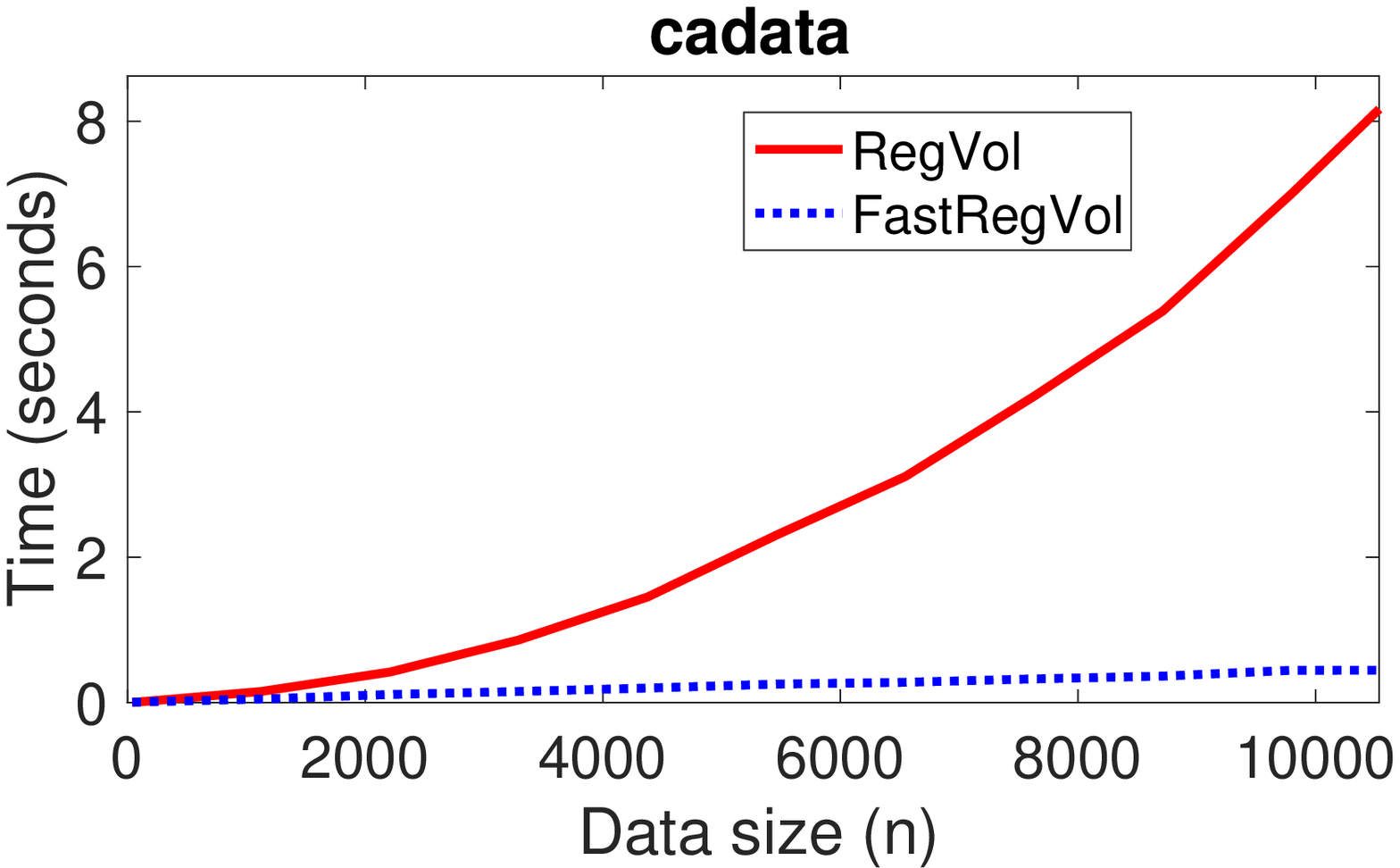}\nobreak
\includegraphics[width=0.5\textwidth]{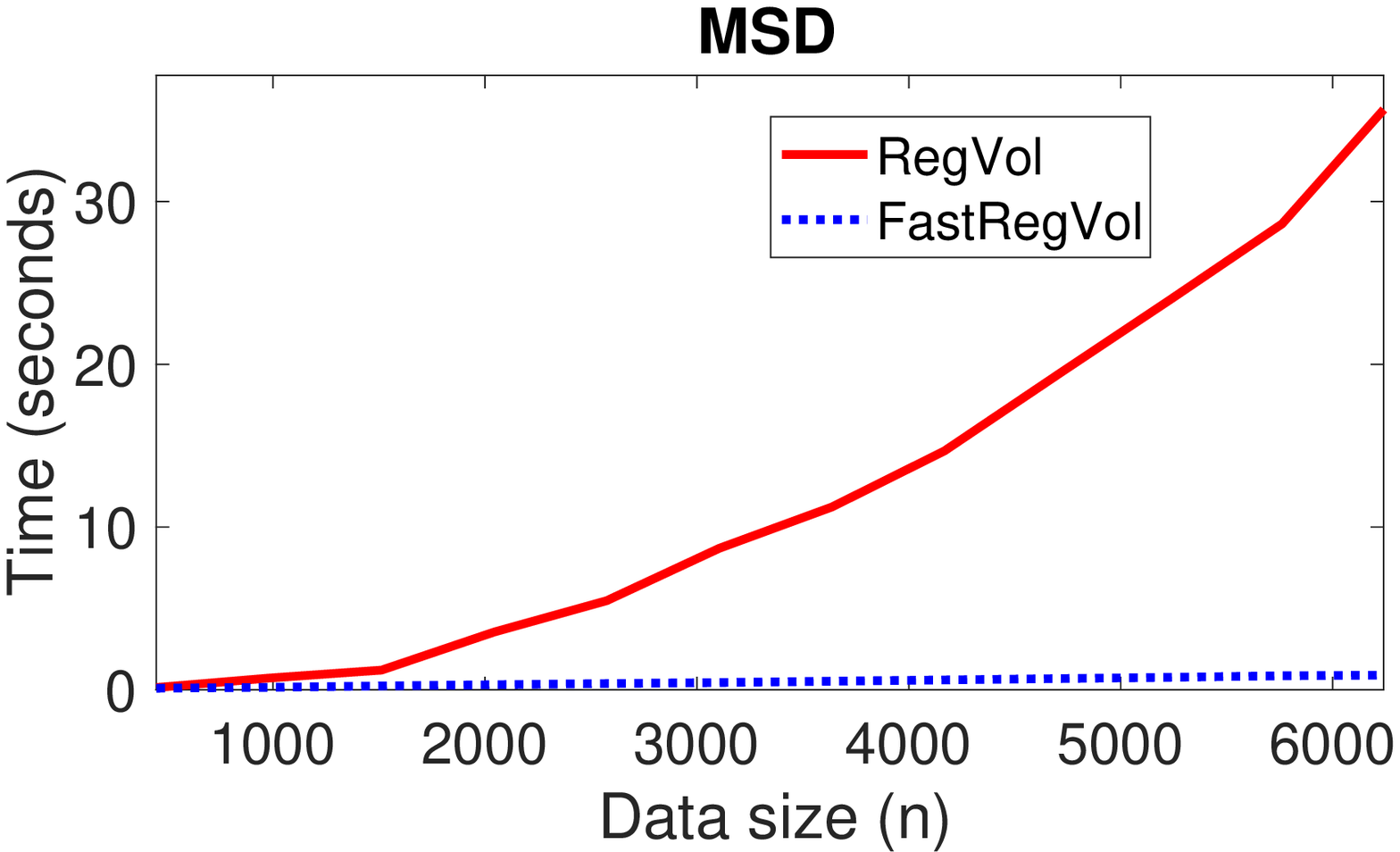}
\includegraphics[width=0.5\textwidth]{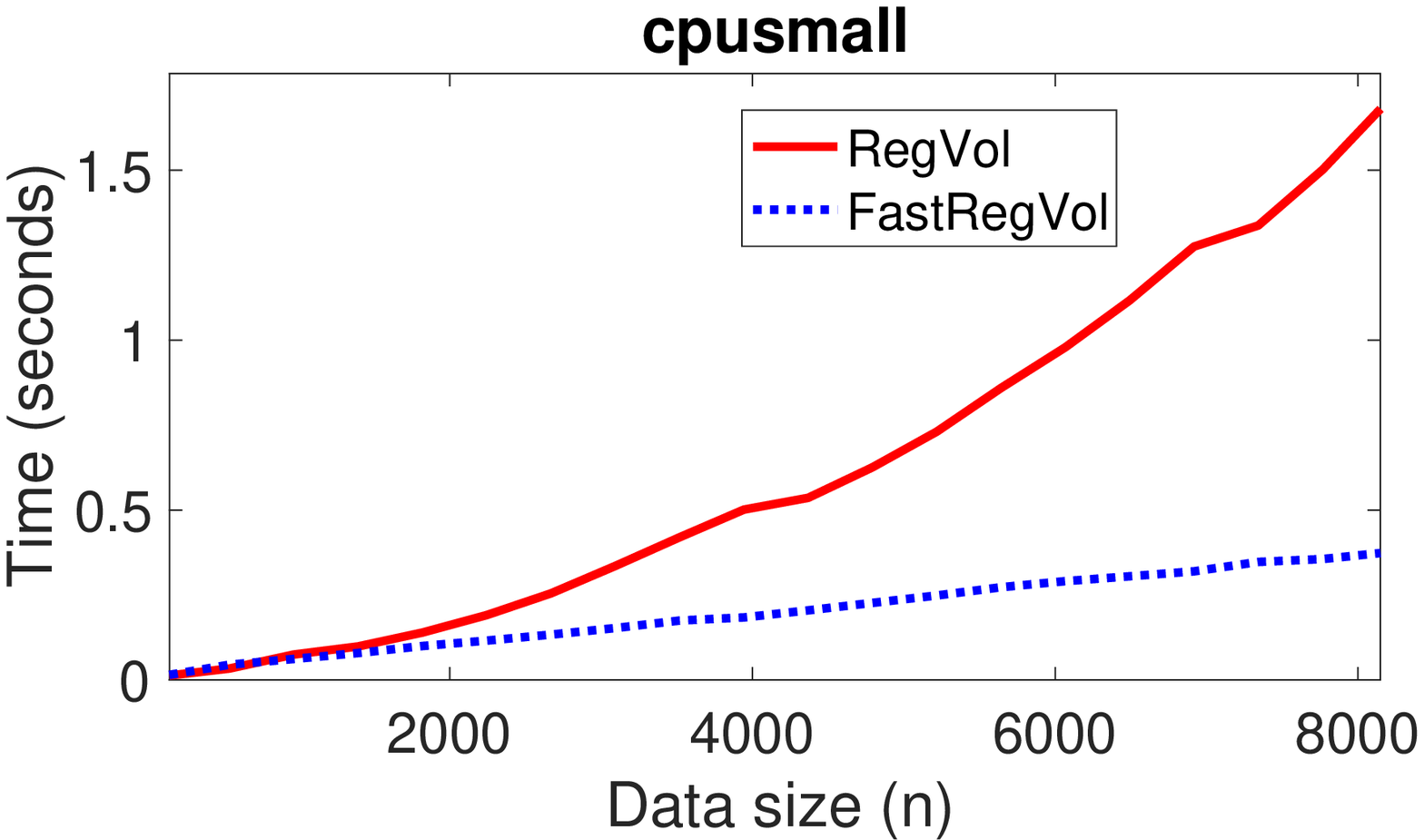}\nobreak
\includegraphics[width=0.5\textwidth]{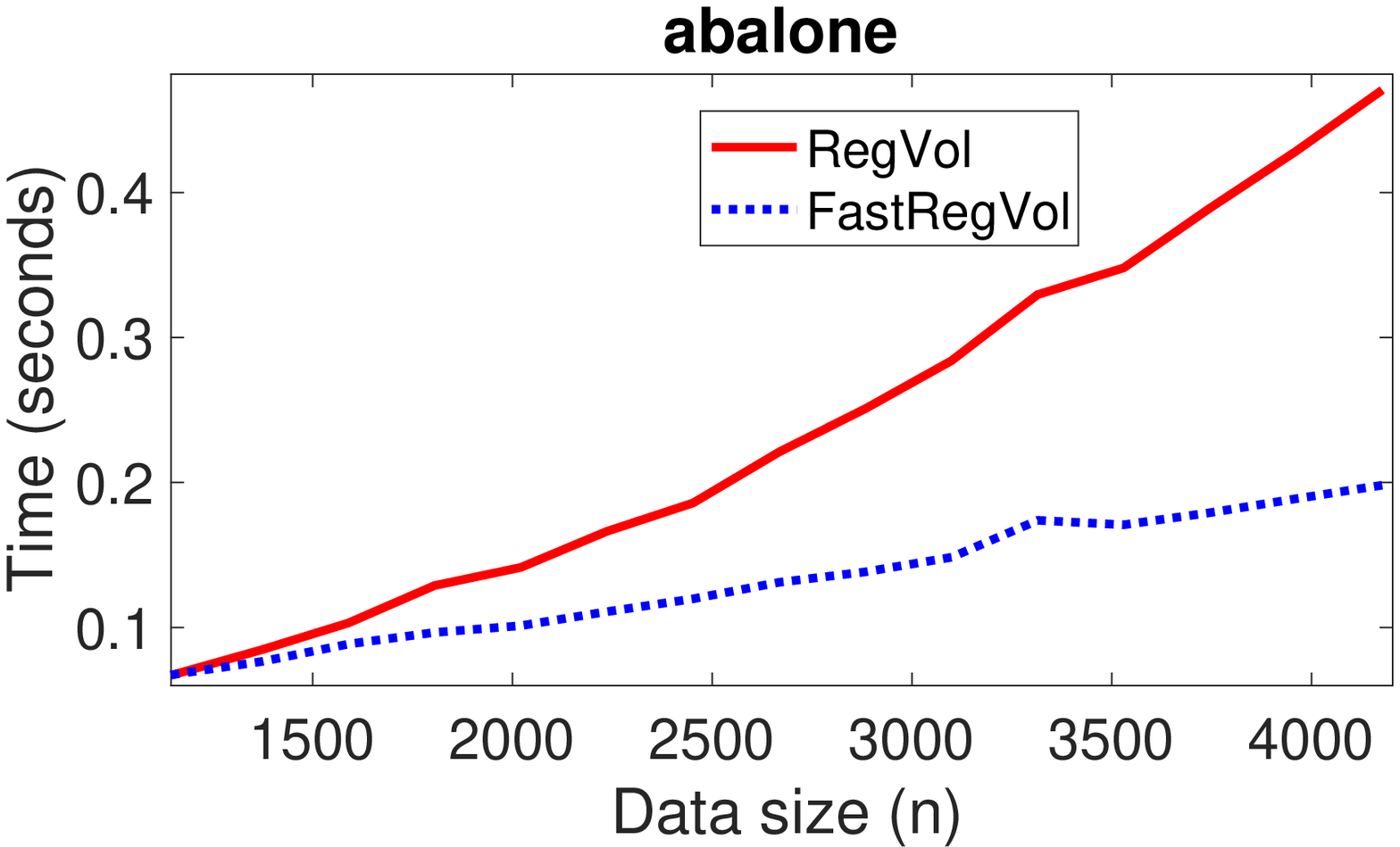}
\caption{Comparison of runtime between {\lazy} 
  and {\volsamp} on four libsvm regression datasets \citep{libsvm}, with the  methods ran on data
  subsets of varying size (n).}
\label{fig:time-plots}
\end{figure}

\subsection{Subset selection for ridge regression}
We applied volume sampling to the task of subset selection for
linear regression, by evaluating the subsampled ridge estimator
$\of{\w_\lambda^*}S$ using the total loss over the full dataset, i.e.
\begin{align*}
\text{Total loss:}&\quad\frac{1}{n}\|\X\of{\w_\lambda^*}S-\y\|^2,\quad
\text{where}\quad\of{\w_\lambda^*}S = (\X_S^\top\X_S+\lambda\I)^{-1}\X_S^\top\y_S.
\end{align*}
We evaluated the estimators for a range of
subset sizes and values of $\lambda$, when the subsets are sampled
according to $\lambda$-regularized volume sampling%
\footnote{Our experiments suggest that using the same
  $\lambda$ for sampling and for computing the ridge estimator works
  best.} 
 and leverage score sampling. The results were averaged over 20 runs of each
experiment. For clarity, Figure \ref{fig:predictions}
shows the results only with one value of $\lambda$ for each
dataset, chosen so that the subsampled ridge estimator performed best
(on average over all samples of preselected size $s$). 
Note that for leverage scores
we did the appropriate rescaling of the instances
before solving for $\of{\w_\lambda^*}S$ 
for the sampled subproblems
(see \cite{randomized-matrix-algorithms} for details). 
Volume sampling does not require any rescaling.
The results on all datasets show that when only a small number of
responses $s$ is obtainable, then regularized volume sampling offers better estimators than
leverage score sampling (as predicted by Theorems \ref{thm:ridge} and \ref{thm:lb-all}). 
The lower-bound from Theorem \ref{thm:lb-all}
part 2 can be observed for dataset cpusmall, where $d=12$ and $d\log
d\approx 30$.

\begin{figure}
\includegraphics[width=0.5\textwidth]{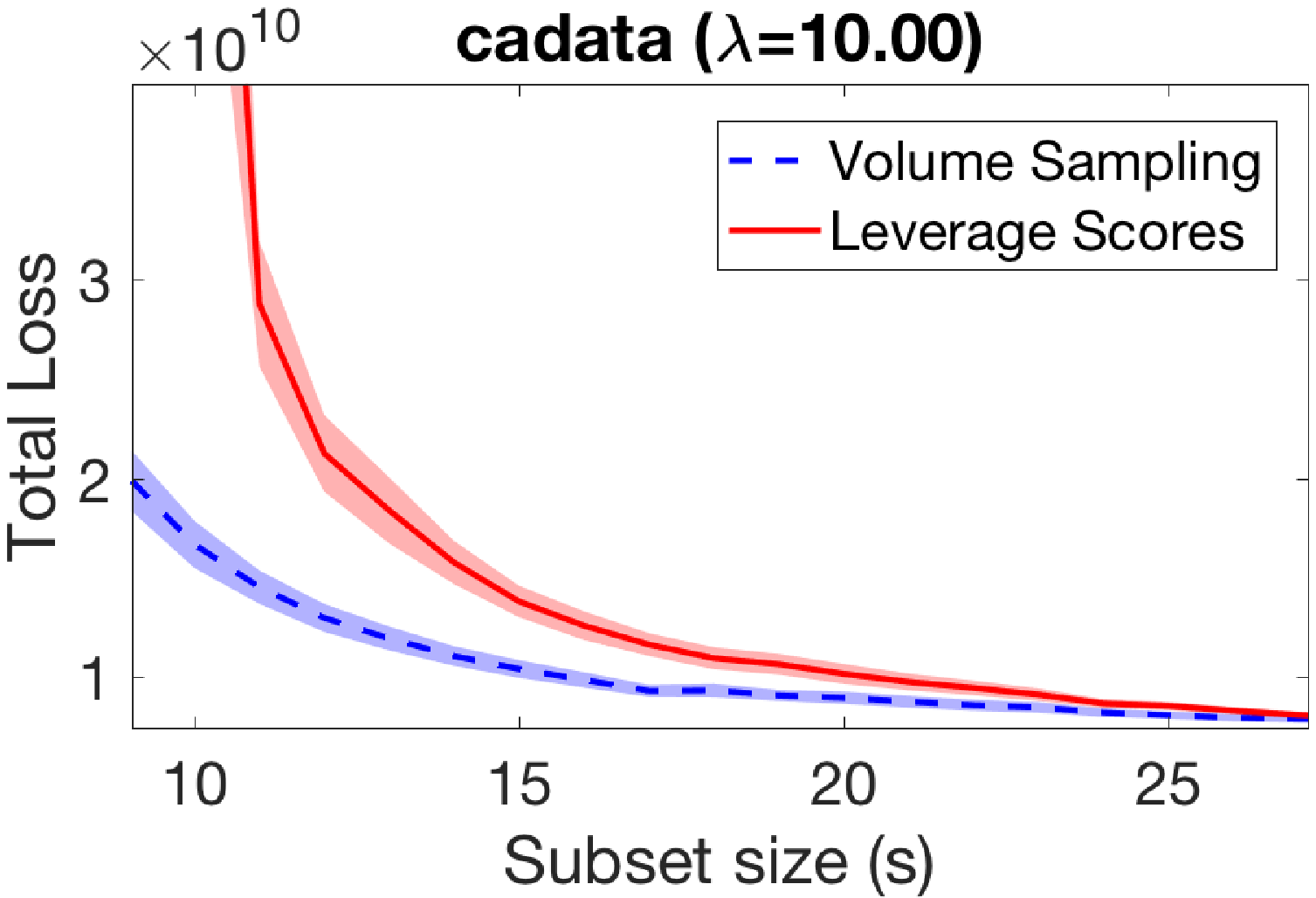}\nobreak
\includegraphics[width=0.5\textwidth]{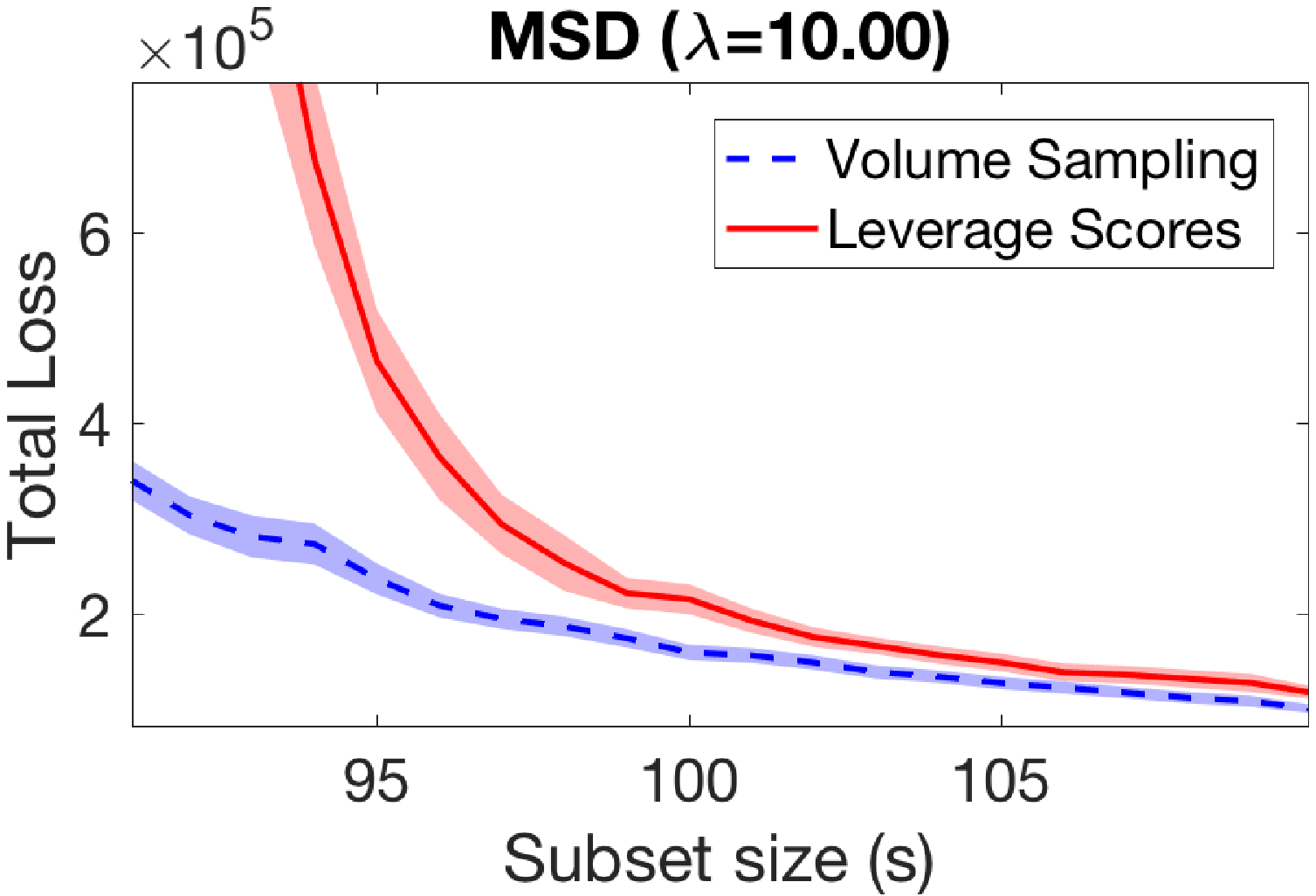}
\includegraphics[width=0.5\textwidth]{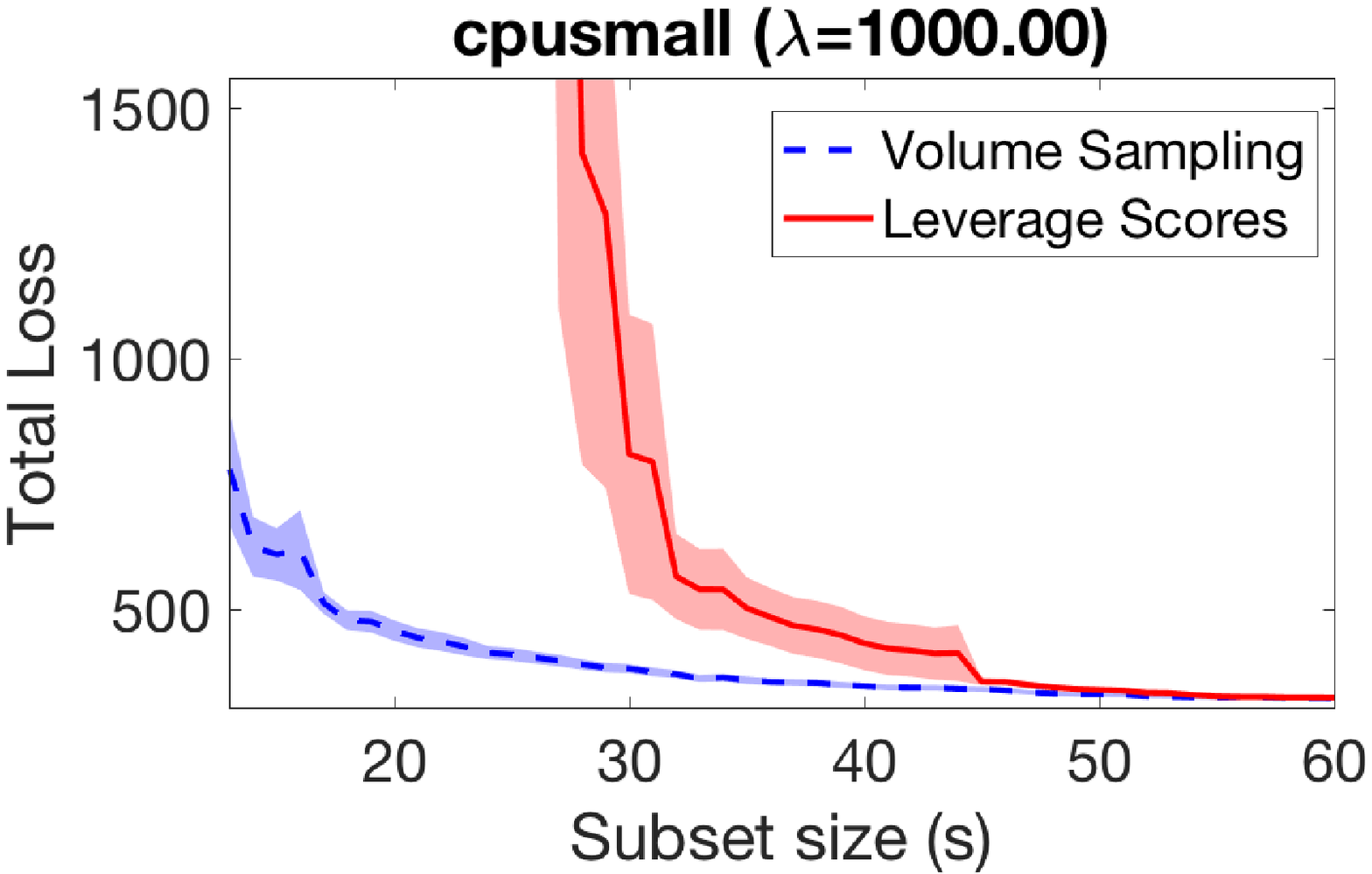}\nobreak
\includegraphics[width=0.5\textwidth]{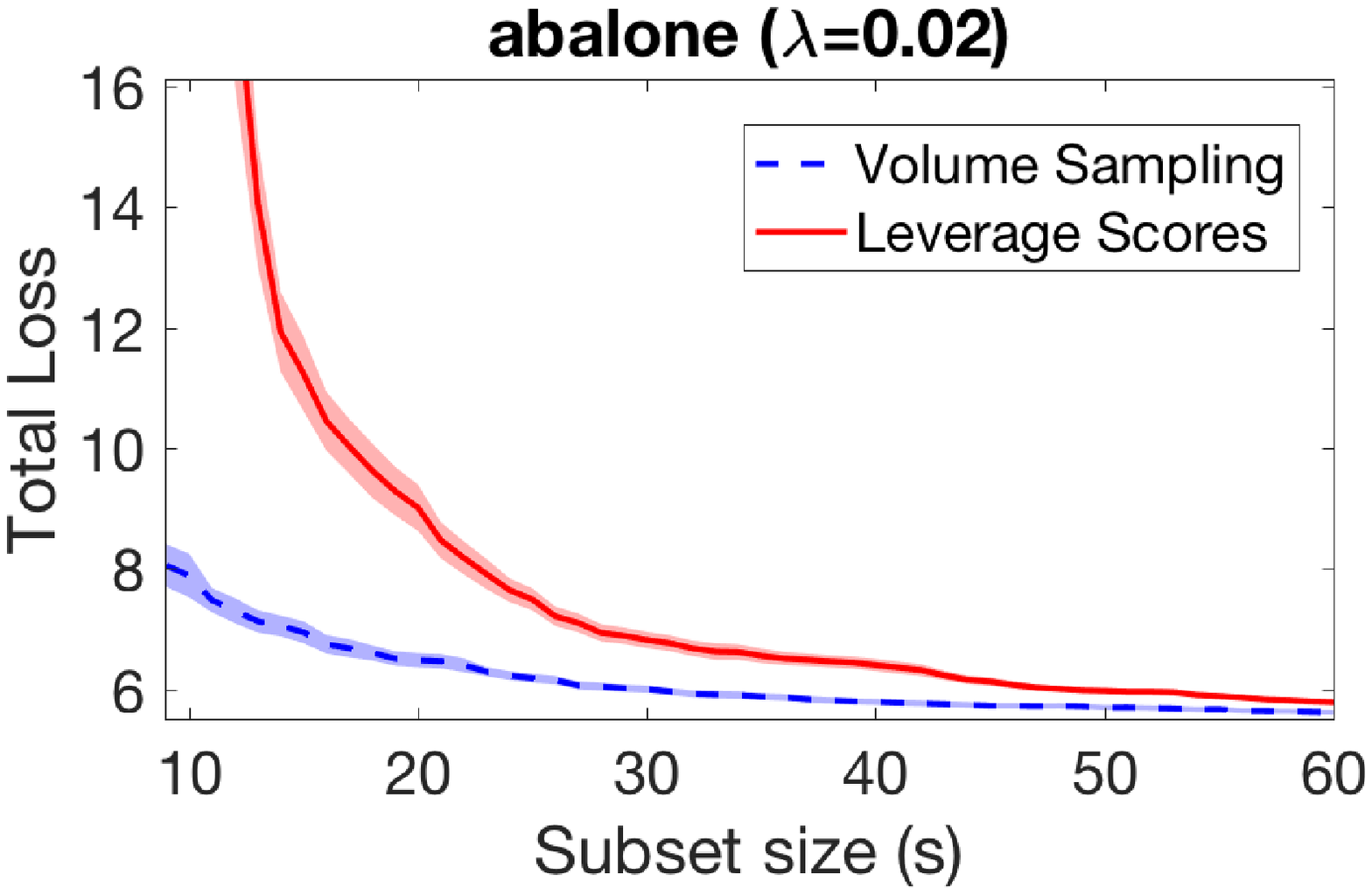}
\caption{Comparison of loss of the subsampled ridge estimator when
  using regularized volume sampling vs using leverage score sampling
  on four datasets.}
\label{fig:predictions}
\end{figure}

\section{Conclusions} \label{sec:conclusions}

Volume sampling is a joint sampling procedures that
produces more diverse samples than i.i.d. sampling.
We developed a method for proving
exact matrix expectation formulas for volume
sampling giving further credence to the fact that this is a
fundamental sampling procedure. 
We also made significant progress on finding an efficient
implementation of this sampling procedure: Our new reverse iterative 
volume sampling algorithm runs in time $O(nd^2)$. Note that 
this running time is within a constant factor of i.i.d. sampling
with exact leverage scores and is a remarkable feat since volume sampling
was only recently shown to be polynomial (that is $O(n^4 s)$ in
\cite{dual-volume-sampling}).

A final long ranging question is how to generalize volume
sampling and the exact matrix expectation formulas to higher order tensors.

\paragraph{Acknowledgments}
Thanks to Daniel Hsu and Wojciech Kot{\l}owski for many
valuable discussions.
This research was supported by NSF grant
IIS-1619271.



\appendix
\section{Inductive proof of Cauchy-Binet}
\label{a:CB}
The most common form of the Cauchy-Binet
equation deals with two real $n\times d$ matrices $\A,\B$:
$\sum_{S\,:\,|S|=d}\det(\A_S^\top\B_S)=\det(\A^\top\B)$. It is
easy to generalize volume sampling and Theorem
\ref{t:vol} to this ``asymmetric'' version. Here we give an
alternate inductive proof.

For $i\in\{1..n\}$, 
let $\a_i,\b_i$ denote the $i$-th row of $\A,\B$, respectively.
For $S\subseteq\{1..n\}$, $\A_S$ consists of all rows indexed
by $S$, and $\A_{-i}$, all except for the $i$-th row.

\begin{theorem}
\label{t:CB}
For $\A,\B \in \R^{n\times d}$ and 
$n-1 \ge s\ge d:$
$$\det(\A^\top\B)
=\frac{1}{{n-d\choose s-d}} \sum_{S\,:\,|S|=s} \det(\A_S^\top\B_S).$$
\end{theorem}

\begin{proof}
$S$ is a size $s$ subset of a set of size $n$.
We rewrite the range restriction $n-1 \ge s\ge d$ for size
$s$ as $1\le n\!-\!s\le n\!-\!d$ and induct on $n-s$.
For the base case, $n\!-\!s = 1$ or $s=n\!-\!1$, we need to show that
$$\det(\A^\top\B)= \frac{1}{n-d} \sum_{i=1}^n \det(\A_{-i}^\top\B_{-i}).
$$
This clearly holds if $\det(\A^\top\B)=0$. Otherwise, by
Sylvester's Theorem
$$ \sum_{i=1}^n
\frac{\det(\overbrace{\A_{-i}^\top\B_{-i}}^{\A^\top\B-\a_i\b_i^\top})}
     {\det(\A^\top\B)}
= \sum_{i=1}^n (1-\a_i^\top (\A^\top\B)^{-1}\b_i)
= n-\overbrace{\tr((\A^\top\B)^{-1}\A^\top\B)}^d.
$$
Induction: Assume $2\le n-s\le n-d$.
\begin{align*}
\det(\A^\top\B)
& \stackrel{\scriptsize
           \begin{array}{c}
             \text{base}\\[-1mm]
             \text{case}\\[-.5mm] 
           \end{array}
           }{=} 
\frac{1}{n-d} \sum_{i=1}^n \det(\A_{-i}^\top\B_{-i})
\\ &\stackrel{\scriptsize
             \begin{array}{c}
               \text{ind.}\\[-1mm]
               \text{step}\\[-.5mm] 
             \end{array}
         }{=} 
\frac{1}{n-d} \sum_{i=1}^n \;\;\sum_{S\,:\,|S|=s,\,i\notin S} 
\frac{1}{{n-1-d\choose s-d}} \det(\A_S ^\top\B_{S})
\\ &\;\;=\;\;\underbrace{\frac{\Blue{n-s}}{n-d} \,\frac{1}{{n-1-d\choose s-d}}}
               _{\frac{1}{{n-d\choose s-d}}}
\; \sum_{S\,:\,|S|=s} \det(\A_S ^\top\B_{S}).
\end{align*}
Note that for the induction step, $S$ is a subset of size
$s$ from a set of size $n-1$ and we have the range
restriction $1\le n\!-\!1\!-\!s\le n\!-\!1\!-\!d$. Clearly, $n\!-\!1\!-\!s$ is one smaller than $n\!-\!s$.
For the last equality, notice that 
each set $S:|S|=s$ is
counted $\Blue{n-s}$ times in the double sum.
\end{proof}

\section{Alternate proof of Theorem \ref{t:einv}}
\label{sec:alternate-proof}
We make use of the following derivative for determinants by \cite{detderiv}:
$$\text{For symmetric $\C$:}\qquad
\frac{\partial \det(\X^\top\C\X)}{\partial \X}
= 2\det(\X^\top\C\X) \C\X(\X^\top\C\X)^{-1}.
$$
The proof begins with generalized Cauchy-Binet for size $s$ volume
sampling:
\begin{align*}
\sum_S \det(\X^\top\I_S\X)&=
{n-d \choose s-d}\det(\X^\top\X).
\intertext{Now, we take a derivative w.r.t. $\X$ on both sides}
\sum_S 2\det(\X^\top\I_S\X)\;\;  (\I_S\X)^{\dagger\top}&=
{n-d \choose s-d}\;\;2\det(\X^\top\X) \;\; \X^{\dagger\top}\\
\Longleftrightarrow\qquad
\underbrace{\sum_S \frac{\det(\X^\top\I_S\X)}
{{n-d \choose s-d}\det(\X^\top\X)} \;\;
(\I_S\X)^{\dagger\top}
           }_{\E [ (\I_S\X)^{\dagger\top}]}
&= \X^{\dagger\top}.
\end{align*}

\section{Proof of Proposition \ref{prop:geometry}}
\label{a:prop-proof}
The main idea behind the proof is to construct variants of
the input matrix $\X$ and relate their volumes.
We use the following standard properties of the determinant:
\begin{proposition}
\label{prop:volume-preserving}
For any matrix $\M$, $\det(\M^\top\M)=\det(\widetilde{\M}^\top\widetilde{\M})$ 
where $\widetilde{\M}$
is produced from $\M$ through the following operations:
\begin{enumerate}
\item 
$\widetilde{\M}$ equals $\M$ except that column
$\m_j$ is replaced by $\m_j + \alpha \m_i$, where $\m_i$ is another
column of $\M$;
\item $\widetilde{\M}$ equals $\M$ except that two rows are swapped.
\end{enumerate}
\end{proposition}
Recall that our goal is to prove the following formula for any
$\X,\y$ and $i\in\n$:
\begin{align*}
\det(\X^\top\X)\,\big(L(\of{\w^*}{-i})-L(\w^*)\big) =
\big(\det(\X^\top\X)-\det(\X_{-i}^\top\X_{-i})\big) \ell_i(\of{\w^*}{-i}).
\end{align*}
By part 2 of Proposition \ref{prop:volume-preserving}, we can assume w.l.o.g. that $i=n$, i.e. that the test
row in Proposition \ref{prop:geometry} is the last row of $\X$. 
As discussed in Section \ref{sec:proof-loss}, the columns of $\X$ are the
feature vectors, denoted by $\f_1,\ldots,\f_d$. Moreover, the optimal
prediction vector on the full dataset,  $\ybh=\X\w^*$, is
a projection of $\y$ onto the subspace spanned by the
features/columns of $\X$, denoted as $\ybh=\P_\X\,\y$.
Let us define a vector $\ybb$ as
\vspace{-5mm}
\begin{align}
\ybb^\top \defeq (\frac{\qquad }{\qquad }\ybh_{-n}^\top
  \frac{\qquad }{\qquad } , y_n), 
\label{eq:ybar}
\end{align}
where $\ybh_{-n}\defeq\X_{-n}\of{\w^*}{-n}$ is the optimal prediction vector for the training
problem $(\X_{-n}, \y_{-n})$. Note, that if
$\rank(\X_{-n})<d$, then $\of{\w^*}{-n}$ may not be
unique, but we can pick any weight vector as long as it minimizes the loss
on the training set $\{1..n\!-\!1\}$.
Next, we show the following claim:
\begin{claim}
\label{c:ybar}
The best achievable loss for the problem $(\X,\y)$ can be decomposed as follows:
\begin{align}
L(\w^*) = L(\of{\w^*}{-n}) - \ell_n(\of{\w^*}{-n}) + \|\ybb - \ybh\|^2.
\end{align}
\end{claim}
\begin{proof}
First, we will show that $\ybb$ is the
projection of $\y$ onto the subspace spanned by all
features and the unit vector $\e_n\in\R^n$ (where $n$ corresponds to the test
row). That is, we want to show that
$\ybb=\P_{(\X,\e_n)}\,\y$. Denote $\widetilde{\y}$ as that 
projection. Observe that $\widetilde{y}_n=y_n$,
because if this was not true, we could construct a vector
$\widetilde{\y} + (y_n-\widetilde{y}_n)\e_n$ that is closer to $\y$ than
$\widetilde{\y}$ and lies
in $\Span(\X,\e_n)$. Thus, the projection does not incur any loss along the
$n$-th dimension and can be reduced to the remaining $n-1$
dimensions, which corresponds to solving the training
problem $(\X_{-n},\y_{-n})$.
Using the definition of $\ybb$ in (\ref{eq:ybar}),
this shows that $\widetilde{\y}=\P_{(\X,\e_n)}\,\y$ equals $\ybb$.

Next, we will show that $\ybh$ is the projection of $\ybb$ onto
$\Span(\X)$, i.e. that $\P_\X\,\ybb=\ybh$. 
By the linearity of projection, we have 
\begin{align*}
    \P_\X\,\ybb &= \P_\X(\ybb - \y + \y) \\
&= \P_\X(\ybb - \y) +
    \P_\X\,\y \\
&= \P_\X(\ybb - \y) + \ybh.
\end{align*}
We already showed that $\ybb=\P_{(\X,\e_n)}\,\y$.  
Therefore, the vector $\ybb-\y$ is orthogonal to the column vectors of
$\X$, and thus $\P_\X(\ybb - \y)=0$. This shows
that $\P_\X\,\ybb=\ybh$. 

Finally, note that since $\ybb$ is the projection of $\y$ onto
$\Span(\X,\e_n)$ and $\ybh\in \Span(\X,\e_n)$, vector $\ybb-\y$ is orthogonal to
vector $\ybb-\ybh$ and by the Pythagorean Theorem we have 
\[\|\ybh - \y\|^2 = \|\ybb - \y\|^2 + \|\ybb - \ybh\|^2.\]
Using the definition of $\ybb$ in (\ref{eq:ybar}), we have 
\[\|\ybb-\y\|^2=\|\ybh_{-n} - \y_{-n}\|^2=L(\of{\w^*}{-n}) -
\ell_n(\of{\w^*}{-n}),\] 
concluding the proof of the claim.
\end{proof}

\begin{proofof}{Proposition}{\ref{prop:geometry}}
We construct a matrix $\Xbb$, adding vector $\ybb$ as an extra
column to
matrix $\X$: 
\begin{align}
\Xbb\defeq (\X\,,\, \ybb)= 
\left( \begin{array}{ccc|c}
&&   & \\
 & \X_{-n}&  & \ybh_{-n}\\
&&    & \\
\hline
& \x_n^\top&&y_n 
  \end{array}\right).
\label{eq:Xbbar}
\end{align}
Applying ``base $\times$ height'' and Claim \ref{c:ybar},
we compute the volume spanned by $\Xbb$: 
\begin{align}
\!\!\det(\Xbb^\top\Xbb) = \det(\X^\top\X)\;
\|\ybb-\ybh\|^2=  \det(\X^\top\X)\;(L(\w^*) -
L(\of{\w^*}{-n}) + \ell_n(\of{\w^*}{-n})).
\label{eq:xbbar-vol} 
\end{align}
Next, we use the fact that volume is
preserved under elementary column operations (Part 1 of Proposition
\ref{prop:volume-preserving}). Note, that prediction
vector $\ybh_{-n}$ is a linear combination of the columns of
$\X_{-n}$, with
the coefficients given by $\of{\w^*}{-n}$. Therefore, looking at the
block structure of $\Xbb$ (see (\ref{eq:Xbbar})), we observe that performing
column operations on the last column of $\Xbb$ with coefficients given
by negative $\of{\w^*}{-n}$, we can zero out that
column except for its last element:
\begin{align*}
\ybb - \X\, \of{\w^*}{-n} = r\,\e_n,
\end{align*}
where $r\defeq y_n - \x_n^\top\of{\w^*}{-n}$ (see transformation (a) 
in (\ref{eq:transformations})). Now, we consider two
cases, depending on whether or not $r$ equals zero. 
If $r\ne 0$, then we further transform the matrix
by a second transformation (b), which zeros out the last
row (the test row) 
using column operations. The entire sequence of operations, 
resulting in a matrix we call $\Xbb_0$, is shown below:
\begin{align}\Xbb=
\left( \begin{array}{ccc|c}
&&   & \\
 & \X_{-n}&  & \ybh_{-n}\\
&&    & \\
\hline
& \x_n^\top&&y_n 
  \end{array} \right) \overset{\text{(a)}}{\rightarrow}
\left( \begin{array}{ccc|c}
&&   & \\
 & \X_{-n}&  & 0\\
&&    & \\
\hline
& \x_n^\top && r
  \end{array} \right) \overset{\text{(b)}}{\rightarrow}
\left( \begin{array}{ccc|c}
&&   & \\
 & \X_{-n}&  & 0\\
&&    & \\
\hline
& 0&&r
  \end{array} \right) =\Xbb_0
\label{eq:transformations}
\end{align}
Note, that due to the
block-diagonal structure of $\Xbb_0$, its volume can be easily described by the
``base $\times$ height'' formula:
\begin{align}
\det(\Xbb_0^\top\Xbb_0) = \det(\X_{-n}^\top\X_{-n})\;r^2 =
\det(\X_{-n}^\top\X_{-n})\;\ell_n(\of{\w^*}{-n}).
\label{eq:xbbar0-vol}
\end{align}
Since $\det(\Xbb^\top\Xbb) = \det(\Xbb_0^\top\Xbb_0)$,
we can combine (\ref{eq:xbbar-vol}) and (\ref{eq:xbbar0-vol}) to
obtain the desired result. 

Finally, if $r=0$ we cannot perform
transformation (b). However, in this case matrix $\Xbb$ has volume
$0$, and moreover, $\ell_n(\of{\w^*}{-n})=r^2=0$, so once again we have
\[\det(\Xbb^\top\Xbb) = 0 =
\det(\X_{-n}^\top\X_{-n})\;\ell_n(\of{\w^*}{-n}),\]
which concludes the proof of Proposition \ref{prop:geometry}.
\end{proofof}


\bibliography{pap}

\end{document}